\documentclass[twoside]{article}

\usepackage[accepted]{aistats2023}
\usepackage{algorithm,algorithmic}

\usepackage{amssymb,bm,amsmath,amsthm,amsfonts,mathrsfs}
\newtheorem{thm}{Theorem}[section]
\newtheorem{lemma}[thm]{Lemma}

\newtheorem{cor}{Corollary}

\newcommand{\mbR}{\mathbb{R}}

\newcommand{\tr}{\rm{Tr}}

\newcommand{\mcL}{\mathcal{L}}
\newcommand{\mcN}{\mathcal{N}}

\newcommand{\mcS}{\mathcal{S}}

\newcommand{\mcR}{\mathcal{R}}

\newcommand{\cov}{{\rm cov}}

\newcommand{\mbE}{\mathbb{E}}
\newcommand{\mbP}{\mathbb{P}}

\usepackage{graphics,graphicx,subcaption,color}
\usepackage[cmintegrals]{newtxmath}
\usepackage{bm}

\begin{document}

%

%

\twocolumn[

\aistatstitle{Mutual Information Learned Regressor: an Information-theoretic Viewpoint of Training Regression Systems}

\aistatsauthor{Jirong Yi \And Qiaosheng Zhang \And Zhen Chen}

\aistatsaddress{University of Iowa \And National University of Singapore \And University of California at Irvine}

\aistatsauthor{Qiao Liu \And Wei Shao \And Yusen He \And Yaohua Wang}
\aistatsaddress{Stanford University \And University of Florida \And Grinnell College \And University of Iowa}

]

\begin{abstract}
	As one of the central tasks in machine learning, regression finds lots of applications in different fields. An existing common practice for solving regression problems is the mean square error (MSE) minimization approach or its regularized variants which require prior knowledge about the models. Recently, Yi et al., proposed a mutual information based supervised learning framework where they introduced a label entropy regularization which does not require any prior knowledge. When applied to classification tasks and solved via a stochastic gradient descent (SGD) optimization algorithm, their approach achieved significant improvement over the commonly used cross entropy loss and its variants. However, they did not provide a theoretical convergence analysis of the SGD algorithm for the proposed formulation. Besides, applying the framework to regression tasks is nontrivial due to the potentially infinite support set of the label. In this paper, we investigate the regression under the mutual information based supervised learning framework. We first argue that the MSE minimization approach is equivalent to a conditional entropy learning problem, and then propose a mutual information learning formulation for solving regression problems by using a reparameterization technique. For the proposed formulation, we give the convergence analysis of the SGD algorithm for solving it in practice. Finally, we consider a multi-output regression data model where we derive the generalization performance lower bound in terms of the mutual information associated with the underlying data distribution. The result shows that the high dimensionality can be a bless instead of a curse, which is controlled by a threshold. We hope our work will serve as a good starting point for further research on the mutual information based regression.
\end{abstract}

\section{Introduction}\label{Sec:Introduction}

The machine learning community has witnessed significant progress ever since the breakthrough made by Krizhevsky et al. \cite{krizhevsky_imagenet_2012} where they proposed a deep convolution neural network, e.g., AlexNet, for image recognition tasks, and the performance of classification on benchmark datasets such as ImageNet has been pushed to above human performance \cite{deng_imagenet:_2009,yu_coca_2022}. Such progress in classification is greatly due to the constant advances over the neural network architecture, the availability of increasingly large dataset, and also the more and more advanced learning algorithms \cite{goodfellow_deep_2016,liu_swin_2021,devlin_bert_2019,zhang_improved_2021,zhu_one-pass_2021,beyer_are_2020}. As another core task in supervised learning, regression also benefits from these advances, and finds applications in many fields such as computer vision and signal processing \cite{he_deliberation_2019,lin_infinitygan_2021,zheng_progressive_2022,bora_compressed_2017,yi_outlier_2018}. For example, in an image synthesis task, the goal can be generating natural images from a random noise vector \cite{makhzani_pixelgan_2017}.

\subsection{Regularizations for Regression}

Despite the various types of regression problems in different applications, the mean square error (MSE) minimization approach and its variants have been the mainstream way for solving it \cite{mohri_foundations_2018,ahuja_linear_2020,loh_high-dimensional_2011,hastie_surprises_2019}. The popularity of MSE minimization is due to the excellent intepretations and intuitions, e.g., minimizing the distance between the predicted labels and the ground truth labels, or maximizing the likelihood of observed data samples \cite{ren_balanced_2022,mohri_foundations_2018,theodoridis_machine_2015}. However, the vanilla MSE approach can fail due to the curse of dimensionality where the number of model parameters exceeds the number of data examples \cite{johnstone_statistical_2009,donoho_high-dimensional_2000,ren_balanced_2022,wu_last_2022}. This is especially ture under the modern deep learningg framework where the number of weights of deep neural networks can scale up to magnitude of trillions, and it is almost unrealistic to collect a dataset of such size \cite{fedus_switch_2021}. Besides, in many challenging tasks, the models with huge number of parameters are necessary to acheive the capability of extracting useful information from data for improving the performance \cite{liu_swin_2021,zheng_progressive_2022}. 

In practice, a commonly used approach for addressing the failure of regression is to incorporate prior knowledge or structure information to make the regression problem more well-posed \cite{mohri_foundations_2018,thrampoulidis_regularized_2015,lin_optimal_2018}. Popular regularizations for regression tasks include the $\ell_2$ norm regularization, $\ell_1$ regularization, and their variants \cite{mohri_foundations_2018,yi_optimal_2021,hastie_surprises_2019}. In $\ell_2$ regularized regression, an extra $\ell_2$ norm term of the model parameter is added to the MSE minimization objective function \cite{hastie_surprises_2019}. This is based on the prior knowledge that the models which overfit training data usually have exceedingly large parameters magnitude \cite{bishop_pattern_2006}. The $\ell_1$ regularized variant, however, adds a $\ell_1$ norm term of the model parameters, which is motivated by assuming that the model parameters are sparse \cite{johnstone_statistical_2009}. These approach have been reported in practice to acheive excellent regression performance, and find many applications such as compressed sensing and cancer treatment planning \cite{bora_compressed_2017,ren_balanced_2022}. However, the prior knowledge may not always be easy to incorporate, especially in scenarios where it is hard to form such prior knowledge \cite{yi_mutual_2022}. 

Recently, Yi et al. investigated the classification task from an information-theoretic viewpoint, and proposed a mutual information based supervised learning framework for training deep learning classifiers \cite{yi_mutual_2022}. Instead of using the mainstream cross entropy training loss objective, they proposed a mutual information learning loss (milLoss), and its equivalent regularization form contains a conditional label entropy term and a marginal label entropy term \cite{yi_mutual_2022}. The most appealing part of their regularization is that it does not require any prior knowledge, and it encourages the model to learn accurately the dependency between the input and the label \cite{yi_mutual_2022}. Their experimental reuslts over benchmark datasets showed that their proposed approach acheived significant improvements over the cross entropy minimization and its other variants \cite{yi_mutual_2022}. This motivates us to investigate the regression task under the mutual information supervised learning framework.

\subsection{Mutual Information Learned Regressor}

In this paper, we propose a mutual information learned regressor (MILR) framework which is based on several observations. First of all, Yi et al. showed that the mutual information learning framework could give significant performance boost in classification tasks when a stochstic gradient descent (SGD) optimization algorithm was used to solve it \cite{yi_mutual_2022}. However, they did not provide a theoretical convergence analysis of the SGD algoritm. Secondly, as we will show in later section, the MSE minimization approach itself is a label conditional entropy learning problem, and it shows the possibility of applying the mutual information learning framework to regression tasks.

Under the MILR framework, based on a difference of entropy formulation we formulate the regression problem from an information-theoretic perspective by using a reparameterization technique \cite{kingma_auto-encoding_2014,preechakul_diffusion_2022}. The key of our approach is to use deep neural networks (DNNs) to learn the parameters of the data distribution so that the conditional differential entropy and the differential entropy can be computed. The regularized form of the proposed formulation contains a MSE term which corresponds to a conditional label differential entropy and another regularization term which corresponds to a label differential entropy. In practice, similar to the classification tasks considered by Yi et al., the SGD algorithm can be used to optimize the DNN parameters \cite{yi_mutual_2022}. In this paper, we give a theoretical convergence analysis of SGD for supervised machine learning tasks under the mutual information learnig framework which is missing in \cite{yi_mutual_2022}.  Based on a matrix concentration inequality, we also give a sample complexity for achieving fast convergence \cite{tropp_introduction_2015,tao_topics_2012}. To better appreciate the motivations of the proposed framework, we introduce a multi-output regression data model for which we derive a regression generalization loss lower bound in terms of the mutual information by using Fano's inequality \cite{cover_elements_2012}. The lower bound gives good characterizations for the connection between the regression generalization performance and the dimensionality. It shows that there exists a threshold on the dependency between the input and the label, above which the high dimensionality of the regression problems can be a bless instead of a curse.

\subsection{Related Works}

Our work is related to the following several lines of works, but there are distinct differences between our work and them \cite{hastie_surprises_2019,ren_balanced_2022,yi_mutual_2022,yi_derivation_2020,yi_towards_2021,yi_trust_2019,li_page_2021,sinha_certifying_2020,zhou_generalization_2018}. First of all, the training objective in our work is closely related to the MSE loss and its variants \cite{mohri_foundations_2018,ren_balanced_2022,thrampoulidis_regularized_2015,hastie_surprises_2019}. Under our framework, the MSE minimization is essentially learning the label conditional differential entropy. This interpretation differs greatly from what the community commonly holds about the MSE minimization, e.g., minmizing the distance between the truth label and the predicted label, or maximizing the probability for the observed data \cite{mohri_foundations_2018,ren_balanced_2022}. The regularized variants of the MSE usually incorporate certain prior knowledge, e.g., sparse model parameters for $\ell_1$ lasso \cite{bhatia_robust_2015,thrampoulidis_regularized_2015}. However, the regularized form of our formulation does not require any such prior knowledge, and it originates from the problem of learning the mutual information associated with the data distribution whose observations are used to train the models \cite{mcallester_formal_2020,yi_mutual_2022}. 

Another highly-related line of work is the information bottleneck viewpoint of machine learning sysems \cite{shwartz-ziv_opening_2017,tapia_information_2020,tezuka_information_2021}. In 2017, Shwartz-Ziv and Tishby applied mutual information tools to investigate the training of DNNs \cite{shwartz-ziv_opening_2017}. They modeled the classification task as a long Markov process with the label variable $Y$ followed by the input variable $X$, and then a sequence of learned representation variables $T_1, T_2, \cdots,T_L$ for an $L$-layer DNN. Under their information bottleneck (IB) framework, they considered the mutual information $I(T_i;X)$ and $I(T_i;Y)$ \cite{shwartz-ziv_opening_2017}. Based on empirical results about the information plane, they argued that training of deep learning classification systems has two phases, e.g., the model will compress the features after an emprical risk minimization stage. This idea was later developed to formulate new training objectives for classification by learning a representation $T$ which contains the most information about the label $Y$ but least information about the input $X$ \cite{amjad_learning_2020,tezuka_information_2021}. However, we consider the mutual information $I(X;Y)$ between the input $X$ and the label $Y$. Besides, our training loss originates from a representation of the $I(X;Y)$, while the that of IB considers $I(X;L) - \beta I(Y;L)$ where $\beta>0$ is a constant parameter \cite{amjad_learning_2020,tezuka_information_2021}. 

Our work is also highly related to \cite{yi_trust_2019,yi_derivation_2020,yi_towards_2021,yi_mutual_2022,wang_robust_2021} where they also considered $I(X;Y)$. Our work differs from them in the following aspects. Firstly, they considered the classification tasks where the label $Y$ is discrete while we consider the regression problem where the label is continuous \cite{yi_trust_2019,yi_mutual_2022,wang_robust_2021}, and the potentially infinite support of the label makes our problem more challenging \cite{yi_mutual_2022,wang_robust_2021}. Secondly, in \cite{yi_trust_2019,yi_towards_2021,wang_robust_2021}, they considered a classification task in a scenario where an adversary exists and tries to attack the models. However, our work is more aligned with that by Yi et al. where they considered the task without any adversaries, and they proposed the mutual information learned classifiers (MILCs) \cite{yi_mutual_2022}. Our work follows the same encoding-decoding paradigm for deep learning classification proposed by Yi et al. \cite{yi_trust_2019,yi_derivation_2020,yi_towards_2021,yi_mutual_2022}, especially \cite{yi_mutual_2022}. The mutual information learning for supervised learning was first proposed by Yi et al. where based on a novel representation of the mutual information, they designed a mutual information learning loss for training deep learning classification systems, and established the sample complexity for training them in practice. In \cite{yi_mutual_2022}, a stochatic gradient descent algorithm was used to minimize the training objective function, but its convergence analysis was missing. Our work extends the mutual information supervised learning framework from classification to regression, and complement the work of Yi et al. by establishing the convergence analysis of the SGD algorithm \cite{yi_mutual_2022}.  

The contributions of this paper are summarized as follows. First of all, we show that the mainstream MSE minimization for regression is equivalent to a differential entropy learning problem, and we propose a mutual information learning framework for training regression systems where the loss objective contains the conditional label differential entropy and the marginal label differential entropy. Besides, we give the theoretical convergence analysis of the stochastic gradient descent algorithm for training regression systems using the new loss objective, and establish the sample complexity for achieving fast convergence. The analysis applies to the loss objective formulated by Yi et al. for the classification task under the same framework \cite{yi_mutual_2022}. Moreover, we introduce the multi-output regression data model, and derive the generalization loss lower bound in terms of the mutual information associated with the data distribution for the models trained over samples drawn from it. All the proofs can be found in the Supplemental Materials.

\section{MSE Minimization as Conditional Differential Entropy Learning}\label{MSEasDifferentialEntropyLn}

We consider the regression task in machine learning where we want to learn a mapping $f: \mbR^n\to \mbR$ from a dataset $\mcS:=\{(x_i,y_i)\}_{i=1}^N$ with elements drawn from $p_{X,Y}\in\mbR^n\times \mbR$. In existing practice, we usually formulate the problem as
$
	\min_\theta \frac{1}{N} \sum_{i=1}^N \ell(f(x_i; \theta), y_i),
$
where the $f(x;\theta)\in\mbR$ is parameterized by $\theta$, and the $\ell:\mbR\times\mbR\to\mbR_+$ is a nonnegative loss function. An example of $\ell(\hat{y},y)$ can be the MSE loss which gives 
\begin{align}\label{Eq:RegressionL2}
	\min_\theta \frac{1}{N} \sum_{i=1}^N \left( f(x_i; \theta) - y_i \right)^2. 
\end{align}
Under certain conditions, the \eqref{Eq:RegressionL2} is equivalent to a maximal likelihood problem formulation of the regression task. To see this, assume the data is generated according to 
$
	Y_i:=f^*(X_i) + E_i, \forall i=1,\cdots,N,
$
and the noise $E_i\in\mbR, \forall i=1, \cdots,N$ are I.I.D. random variables according to $\mcN(0,1)$ where $f^*:\mbR^n\to\mbR$ is a ground truth mapping. In this setting, the conditional distribution $p_{Y_i|X_i}$ will be $\mcN(f^*(X_i),1), i=1,\cdots,N$ which are conditionally independent. The $X_i,i=1,\cdots,N$ are also assumed to be I.I.D. according to $p_{X}\in\mbR^n$, thus the joint distribution $p_{Y_1,\cdots,Y_N|X_1,\cdots,X_N}$ of becomes
\begin{align*}
	p(Y_1,\cdots,Y_N|X_1,\cdots,X_N)
	& = \prod_{i=1}^N p(Y_i|X_i) \\
	&  = \prod_{i=1}^N \frac{\exp\left(-\frac{1}{2} (Y_i - f^*(X_i))^2 \right)}{ \sqrt{2\pi} }.
\end{align*}
When the $f(x;\theta)$ is used to approximate $f^*$ and the dataset $\mcS$ is used to determine $\theta$, we can find an optimal $\theta^*$ by minimizing the negative logarithm likelihood of $p(Y_1,\cdots,Y_N|X_1,\cdots,X_N)$ at realization $\{(x_i,y_i)\}_{i=1}^N$, i.e., 
\begin{align}
	& \min_\theta - \frac{1}{N} \log\left( \prod_{i=1}^N \frac{1}{ \sqrt{2\pi} } \exp\left(-\frac{1}{2} (y_i - f(x_i; \theta))^2 \right) \right) \\
	& \Leftrightarrow \min_\theta - \frac{1}{N} \sum_{i=1}^N \log\left( \frac{1}{\sqrt{2\pi} } \exp\left(-\frac{1}{2} (y_i - f(x_i; \theta))^2 \right) \right), \label{Eq:LblEtrpLn} \\
	& \Leftrightarrow \min_\theta \frac{1}{2N} \sum_{i=1}^N (y_i - f(x_i;\theta))^2 \label{Eq:MSE_eqvlt}.
\end{align} 
The \eqref{Eq:MSE_eqvlt} is equivalent to \eqref{Eq:RegressionL2}, while \eqref{Eq:LblEtrpLn} essentially learns the label conditional differential entropy as derived from Theorem \ref{Thm:CndtDiffEtrpLnViaCrsDiffEtrpMin}, \ref{Thm:DiffEtrpLnViaCrsEtrpMin} and their implications.

\begin{thm}\label{Thm:CndtDiffEtrpLnViaCrsDiffEtrpMin}
	(Conditional Differential Entropy Learning via Conditional Cross Entropy Minimization) For an arbitrary joint distribution $p_{X,Y}$ of two continuous random variables or vectors $X$ and $Y$, we have
	$h(Y|X)\leq \inf_{q_{Y|X}} H(p_{Y|X}, q_{Y|X})$ 
	where the conditional cross differential entropy is defined as
	\begin{align}\label{Eq:CndtCrsDiffEntrp}
		H(p_{Y|X}, q_{Y|X}):=\int_{x,y} p_{X,Y}(x,y) \log\left( \frac{1}{q_{Y|X}(y|x)} \right)dxdy.
	\end{align} 
	The equality holds if and only $q_{Y|X}=p_{Y|X}$. Moreover, let $\hat{P}_{Y|X}$ be an empirical conditional distribution of $Y$ and $\hat{P}_X$ be an empirical distribution of $X$, and define 
	$
		R_{Y|X}:=\frac{p_{Y|X}}{\hat{p}_{Y|X}}, R_X:= \frac{p_X}{\hat{p}_X}.
	$
	Then 
	\begin{align}\label{Eq:CndtCrsDiffEntrp_Emprc}
		h(Y|X)\leq \inf_{q^g_{Y|X}} H(p^g_{Y|X}, q^g(Y|X)),
	\end{align}
	where 
	$
		p^g_{X,Y}:=p^g_X p^g_{Y|X}, 
		p^g_X:=R_Xp_X, 
		p^g_{Y|X}:=R_{Y|X}p_{Y|X}$, and
$		q^g_{Y|X}:=q_{Y|X}/R_{Y|X}.
	$
	The equality holds if and only if $p_Y = \hat{p}_Y = q_Y$.
	
\end{thm}

Theorem \ref{Thm:CndtDiffEtrpLnViaCrsDiffEtrpMin} implies that the the conditional differential entropy $h(Y|X)$ can be estimated via solving the conditional cross differential entropy minimization problem \eqref{Eq:CndtCrsDiffEntrp_Emprc}. It also shows the possibility of learning $h(Y|X)$ from data points sampled from $p_{X,Y}$ since the problem \eqref{Eq:CndtCrsDiffEntrp_Emprc} only involves the empirical distribution when we assume $R_X(x)=1, \forall x$ and $R_{Y|X}(y|x)=1, \forall x,y$. Based on this, we can see that the \eqref{Eq:LblEtrpLn} essentially estimates the $h(Y|X)$ with the probability mass $\hat{P}_X(x)=\frac{1}{N}$ being uniform distribution, the probability mass $\hat{P}_{Y|X}$ being one-hot distribution, and $q_{Y|X} (y_i|x_i)=\frac{1}{\sqrt{2\pi} } \exp\left(-\frac{1}{2} (y_i - f(x_i; \theta))^2\right)$. Similarly, the differential entropy $h(Y)$ can also be learned from empirical distributions, and the result is formally presented in Theorem  \ref{Thm:DiffEtrpLnViaCrsEtrpMin}. 

\begin{thm}\label{Thm:DiffEtrpLnViaCrsEtrpMin}
	(Differential Entropy Learning Via Cross Entropy Minimization) For a continuous random variable $Y\sim p_Y(Y)$, we have
	$
		h(Y)\leq \inf_{q_Y(Y)} H(p_Y(Y), q_Y(Y)),
	$
	where $q_Y(Y)$ is a valid distribution of $Y$, and the $H(p_Y(Y), q_Y(Y))$ is the cross entropy. The equality holds if and only if $q_Y = p_Y$. Let $\hat{p}_Y(Y)$ be an empirical distribution of $Y$, and define $R_Y(y) = \frac{p_Y(y)}{\hat{p}_Y(y)}, \forall y$. Then 
	$
		h(Y) \leq \inf_{q^g_Y(Y)} H^g(p^g_Y(Y), q^g_Y(Y)),
	$
	where 
	\begin{align}
		q^g_Y(y) := \hat{p}_Y(y)R_Y(y), p^g_Y(y) := q_Y(y)/R_Y(y), \forall y,
	\end{align}
	and 
	\begin{align}
		H^g(p^g_Y(Y), q^g_Y(Y)):=\int_y p^g_Y(y) \log\left( \frac{1}{q^g_Y(y)} \right) dy.
	\end{align}
	The equality holds if and only if $p_Y = \hat{p}_Y = q_Y$.
\end{thm}

\section{Mutual Information Learning Loss for Regression}

In this section, we introduce the mutual information learning regression (MILR) framework based on the observations from previous sections and the work by Yi et al. \cite{yi_mutual_2022} where the authors proposed a mutual information learned classifier (MILC) framework for training classification systems using mutual information objective, and showed significant performance gains when comapred with the cross entropy training loss \cite{yi_mutual_2022}. 

For a joint distribution $p_{X,Y}$ in $\mbR^n\times \mbR$, the mutual informaton $I(X;Y)$ can be computed via 
\begin{align}
	I(X;Y) = h(Y) - h(Y|X),
\end{align}
where $h(Y)$ is the differential entropy and $h(Y|X)$ is the conditional differential entropy, i.e.,
\begin{align}\label{Eq:DifferentialEntropies}
	h(Y) = - \int_Y p_Y (y) \log(p_Y(y)) dy, \\
	h(Y|X) = - \int_{X,Y} p_{X,Y}(x,y) \log(p_{X,Y}(y|x))dxdy.
\end{align}

From Theorem \ref{Thm:DiffEtrpLnViaCrsEtrpMin}
\begin{align}\label{Eq:GeneralForm}
	h(Y) \leq \inf_{q_Y} H(p_Y, q_Y), \\
	h(Y|X) \leq \inf_{q_{Y|X}} H(p_{Y|X}, q_{Y|X}).
\end{align}
Similar to \cite{yi_mutual_2022}, we can parameterize $q_{Y|X}, q_{Y}$ with two neural networks whose weights are $\theta_{Y|X}, \theta_Y$, respectively. This gives us
\begin{align}\label{Eq:ParametricForm}
	h(Y) \leq \inf_{\theta_Y \in\Gamma} H(p_Y, q_Y(Y; \theta_Y)), \\
	h(Y|X) \leq \inf_{\theta_{Y|X} \in\Theta} H(p_{Y|X}, q_{Y|X}(Y|X; \theta_{Y|X})),
\end{align}
where $\Theta$ and $\Gamma$ are the searching space of $\theta_{Y|X}$ and $\theta_Y$, respectively. In \cite{yi_mutual_2022}, Yi et al. showed that under certain conditions, the ground truth mutual information can be well approximated by 
\begin{align}\label{Defn:MIParametricDistributional}
	I_{\Theta, \Gamma}
	& := \inf_{\theta_Y \in\Gamma} \mbE_{p_Y} \left[\log \frac{1}{q_Y(Y; \theta_Y)} \right] \nonumber \\
	& \quad - \inf_{\theta_{Y|X} \in\Theta} \mbE_{p_{X,Y}} \left[\log\frac{1}{q_{Y|X}(Y|X; \theta_{Y|X})} \right].
\end{align}

In practice, since we do not have access to the ground truth distributions $p_{X,Y}$ and $p_{Y}$, we need to train the regression system using empirical distributions $\hat{p}_{X,Y}$ and $\hat{p}_{Y}$. More specifically, the estimate of conditional differential entropy $\hat{h}(Y|X)$ can be obtained via
\begin{align}\label{Eq:LrndCdtnDfrtEtrp}
	\hat{h}(Y|X):= \inf_{\theta_{Y|X}} \frac{1}{N}\sum_{i=1}^N \log\left( \frac{1}{q_{Y|X; \theta_{Y|X}}(y_i|x_i)} \right).
\end{align}
where we used the uniform distribution mass function as the empirical distribution $\hat{P}_X(X)$, and the one-hot encoding for the empirical conditional label distribution $\hat{P}_{Y|x_i}$ \cite{yi_mutual_2022}. Similarly, 
\begin{align}
	\hat{h}(Y):=\inf_{\theta_Y} \frac{1}{N}\sum_{i=1}^N \log\frac{1}{q_Y(y_i;\theta_Y)}.
\end{align}

\subsection{Weight Sharing}

In practice, the weight sharing techniques are frequently used to reduce computational complexity such as the kernels in convolutional neural networks \cite{goodfellow_deep_2016}. It can also be applied in our fomrulation, e.g., when only one neural network is allowed and the weights are shared, the learned marginal label distribution can be calculated via
\begin{align}\label{Eq:LrndMrgnlLblDstrbt}
	q_{Y}(y; \theta_{Y|X}):= \frac{1}{N} \sum_{i=1}^N  q_{Y|X}(y|x_i; \theta_{Y|X}), 
\end{align}
then the label entropy can be estimated via
\begin{align}\label{Eq:LrndLblDfrtlEtrp}
	\hat{h}(Y)
	& : = \inf_{\theta_{Y|X}} \frac{1}{N} \sum_{j=1}^N \log\left( \frac{1}{q_{Y}(y_j; \theta_{Y|X})} \right) \\
	& = \inf_{\theta_{Y|X}} - \frac{1}{N} \sum_{j=1}^N \log\left( \frac{1}{N} \sum_{i=1}^N  q_{Y|X}(y_j|x_i; \theta_{Y|X}) \right).
\end{align}

Thus the mutual information can be learned via solving
\begin{align}\label{Eq:MutuInfoEmprcFrm}
	\hat{I}^{(N)}(X;Y)
	& := \hat{h}(Y) - \hat{h}(Y|X) \\
	& = \inf_{\theta_{Y|X}} - \frac{1}{N} \sum_{j=1}^N \log\left( \frac{1}{N} \sum_{i=1}^N  q_{Y|X}(y_j|x_i; \theta_{Y|X}) \right) \\
	&\quad - \inf_{\theta_{Y|X}} \frac{1}{N}\sum_{i=1}^N \log\left( \frac{1}{q_{Y|X}(y_i|x_i; \theta_{Y|X})} \right)
\end{align}
whose equivalent regularized form can be 
\begin{align}\label{Eq:MutuInfoEmprcFrm_rglrzd}
	& \inf_{\theta_{Y|X}} \frac{1}{N}\sum_{i=1}^N \log\left( \frac{1}{q_{Y|X}(y_i|x_i; \theta_{Y|X})} \right) \nonumber \\
	& \quad + \lambda_{ent} \left(  - \frac{1}{N} \sum_{j=1}^N \log\left( \frac{1}{N} \sum_{i=1}^N  q_{Y|X}(y_j|x_i;\theta_{Y|X}) \right)  \right)
\end{align}
where $\lambda_{ent}>0$. The \eqref{Eq:MutuInfoEmprcFrm_rglrzd} has similar interpretation as the counterpart in classification tasks proposed by Yi et al. \cite{yi_mutual_2022}, i.e., we encourage the learning of a model which can reduce the uncertainty of the label $y_i$ when its corresponding input $x_i$ is given (as indicated by the first term in \eqref{Eq:MutuInfoEmprcFrm_rglrzd}), and also accurately capture the label marginal distribution (as indicated by the second term in \eqref{Eq:MutuInfoEmprcFrm_rglrzd}). Similarly, when we consider the \eqref{Defn:MIParametricDistributional}, the corresponding regularized form will be 
\begin{align}\label{Defn:MIasRegularization_ParametricDistributionalTwo}
	& \inf_{\theta_Y \in\Gamma, \theta_{Y|X} \in\Theta} \lambda_{ent}\mbE_{p_Y} \left[\log \frac{1}{q_Y(Y; \theta_Y)} \right] \nonumber \\
	&\quad + \mbE_{p_{X,Y}} \left[\log\frac{1}{q_{Y|X}(Y|X; \theta_{Y|X})} \right].
\end{align}
In the case where the weight sharing is used, we have
\begin{align}\label{Defn:MIasRegularization_ParametricDistributionalSingle}
	& \inf_{\theta_{Y|X} \in\Theta} \lambda_{ent}\mbE_{p_Y} \left[\log \frac{1}{q_Y(Y; \theta_{Y|X})} \right] \nonumber \\
	&\quad  + \mbE_{p_{X,Y}} \left[\log\frac{1}{q_{Y|X}(Y|X; \theta_{Y|X})} \right],
\end{align}
where $q_Y(Y;\theta_{Y|X})$ is defined as 
\begin{align}
	q_Y(Y;\theta_{Y|X}) :=\int_x p_X(x)q_{Y|X}(Y|x; \theta_{Y|X}) dx
\end{align}

The $q_{Y|X}$ is discrete and usually has finite support set in classification tasks while it is continuous and usually has infinite support set in regression tasks considered in the paper. This fundamental difference results in that we cannot directly apply what Yi et al. designed for classification tasks, i.e., adding softmax layer on top of a deep neural network logit output to get the label probability mass distribution \cite{yi_mutual_2022}. Instead, we follow a re-parameterization approach to estimate the continuous label distribution similar to \cite{kingma_auto-encoding_2014,preechakul_diffusion_2022}. We refer to the proposed framework as mutual information learned regression (MILR) framework which is illustrated in the Supplemental Materials where we present the training and the inference pipeline under the MILR framework. In this paper, we will focus on the case without weight sharing, and leave the weight sharing case for future work.

\section{Stochastic Gradient Descent for Optimizing Mutual Information Learning Loss}

In this section, we consider the stochastic gradient descent (SGD) algorithm for solving an unconstrained problem, i.e., 
\begin{align}\label{Defn:Unconstrained}
	\inf_{\theta_Y, \theta_{Y|X}} \mcL(\theta),
\end{align}
where
\begin{align}\label{Defn:DistributionalLoss}
	\mcL(\theta)
	& := \lambda_{ent}\mbE_{p_Y} \left[\log \frac{1}{q_Y(Y; \theta_Y)} \right] \\
	& \quad + \mbE_{p_{X,Y}} \left[\log\frac{1}{q_{Y|X}(Y|X; \theta_{Y|X})} \right], 
\end{align}
and $\theta:=[\theta_{Y|X}^T \theta_Y^T]^T\in\mbR^{m+m'}$ with $\theta_{X|Y}\in\Theta\subset\mbR^m$ and $\theta_Y\in\Gamma\in\mbR^{m'}$. The gradient of $\mcL(\theta)$ can be computed as
\begin{align}\label{Eq:gradOfmcL}
	\nabla \mcL(\theta)
	= \left[\begin{matrix}
		\nabla_{\theta_{Y|X}} \mcL(\theta) \\ 
		\nabla_{\theta_Y} \mcL(\theta)
	\end{matrix}\right] 
	= \left[\begin{matrix}
		\mbE_{p_{X,Y}}\left[ - \frac{\nabla_{\theta_{Y|X}} q_{Y|X}(Y|X;\theta_{Y|X})}{q_{Y|X}(Y|X;\theta_{Y|X})} \right] \\ 
		\lambda_{ent} \mbE_{p_Y} \left[ - \frac{\nabla_{\theta_Y} q_Y(Y;\theta_Y)}{q_Y(Y;\theta_Y)} \right]
	\end{matrix}\right] 
\end{align}
The SGD updating rules are presented in Algorithm \ref{Alg:SGD} where in each iteration $t$, we randomly sample a batch of $N$ data points $\{(x_i^t,y_i^t)\}_{i=1}^N$ with $(x_i,y_i)$ being I.I.D. according to $p_{X,Y}\in\mbR^n\times\mbR$. An estimate of the gradient $\nabla \mcL(\theta_t)$ at $\theta_t$ from the sample batch will be used to update the model parameters $\theta$, i.e., $\theta_{t+1}:=\theta_t - \eta_t  \nabla \mcL^{(N)}(\theta_t)$ where $\nabla \mcL^{(N)}(\theta_t)$ is defined in Algorithm \ref{Alg:SGD}. Starting from a given intialization $\theta_0\in\mbR^{m+m'}$, the process continues until it converges. In Theorem \ref{Thm:cnvgcRt}, we give a convergence analysis of Algorithm \ref{Alg:SGD}.

\begin{algorithm} 
	\caption{Stochastic Gradient Descent for Solving \eqref{Defn:Unconstrained}} 
	\label{Alg:SGD} 
	\begin{algorithmic} 
		\REQUIRE Learning rate $\{\eta_t\}_{t=0}^{T-1}$ 
		\ENSURE Initialization $\theta_0\in\mbR^{m+m'}$
		\FOR{$t=0,1,\cdots,T-1$}
		\STATE Sample $\mcS_{t+1}:=\{(X_i^{t+1}, Y_i^{t+1})\}_{i=1}^N$ with I.I.D. $(X_i^{t+1}, Y_i^{t+1}) \sim p_{X,Y}\in\mbR^{n}\times\mbR$
		\STATE $\nabla \mcL^{(N)}(\theta_t):=\frac{1}{N} \sum_{i=1}^N \nabla	\ell(\theta_t; X_i^{t+1},Y_i^{t+1})$ with 
		\begin{align}
			\ell(\theta_t; X_i^{t+1},Y_i^{t+1})
			& :=\lambda_{ent} \log\frac{1}{q_Y(Y_i^{t+1}; \theta_Y^{t})}  \nonumber \\
			&\quad + \log\frac{1}{q_{Y|X}(Y_i^{t+1}|X_i^{t+1}; \theta_{Y|X}^t)}
		\end{align}
		\STATE $\theta_{t+1} = \theta_t - \eta_t \nabla \mcL^{(N)}$
		\ENDFOR
	\end{algorithmic}
\end{algorithm}

\begin{thm}\label{Thm:cnvgcRt}
	(Convergence Guarantees of Stochastic Gradient Descent for Solving \eqref{Defn:Unconstrained}) We consider the problem defined in  \eqref{Defn:Unconstrained}, and assume the $\mcL(\theta)$ is $s$-smooth, i.e., for a constant $s>0$, 
	\begin{align}\label{Defn:s-smooth}
		\|\nabla \mcL(\theta) - \nabla \mcL(\theta')\| \leq s\|\theta - \theta'\|, \forall \theta, \theta'\in\mbR^{m+m'}.
	\end{align}
	At each iteration of Algorithm \ref{Alg:SGD} for solving \eqref{Defn:Unconstrained}, let the step size $\eta_t\in\left(0,\frac{2}{s}\right)$. Define $\Delta(\theta_0):= \mcL(\theta_0) - \inf_\theta \mcL(\theta)$ where $\theta_0\in\mbR^{m+m'}$ is an initialization. Then, if 
	\begin{align} \label{Eq:IterationComplexity}
		T\geq \frac{\Delta(\theta_0)}{\alpha \epsilon} + \frac{2}{s\alpha \epsilon} \sum_{t=0}^{T-1} \mbE[\|\nabla\mcL^{(N)}(\theta_t) - \nabla \mcL(\theta_t) \|^2]
	\end{align} 
	where 
	\begin{align} 
		\alpha:=\min_{i=0,1,\cdots,T-1} -(\frac{1}{2} s \eta_t^2 - \eta_t), 
	\end{align} 
the expectation is with respect to $\theta_t, S_{t+1}, t=1,\dots,N$, we have $\frac{1}{T} \sum_{t=1}^{T-1} \mbE\left[|| \nabla_{\theta} \mcL(\theta_t)\|\right] \leq \epsilon$ where $\epsilon>0$ is a constant and the expectation is with respect to $\theta_t, t=1,\cdots,N$.
\end{thm}

Theorem \ref{Thm:cnvgcRt} shows that when the Algorithm \ref{Alg:SGD} is applied to solve \eqref{Defn:Unconstrained} in an online setting, it can converge to an stationary point of $\mcL(\theta)$. Besides, the number of iterations $T$ needed for the convergence depends on the initialization $\theta_0$ and the estimation of gradient at each iteration. For example, when the initialization is close to an optimal solution (i.e., $\Delta(\theta_0)$ is small), and the gradient estimate is accurate in each iteration (i.e., $\mbE\left[ \|\nabla \mcL^{(N)}(\theta_t) - \nabla \mcL(\theta)\|^2 \right]$ is small), the SGD in Algorithm \ref{Alg:SGD} can converge in less number of iterations. 

Next, we will show that $\mbE\left[ \|\nabla \mcL^{(N)}(\theta_t) - \nabla \mcL(\theta)\|^2 \right]$ will be very small when the size $N$ of the sample set $\mcS$ is large enough. The result is formally presented in Theorem \ref{Thm:GradConcentrationIneq} and its implications. Before giving Theorem \ref{Thm:GradConcentrationIneq}, we first introduce a concentration inequality for random matrices in Lemma \ref{Lem:Conineq_Log-Prob-Grad} which will be used to establish Theorem \ref{Thm:GradConcentrationIneq}.

\begin{lemma}\label{Lem:Conineq_Log-Prob-Grad}
	(Concentration Inequality of Logarithm-probability Loss Function Gradient) Let $U_1,\cdots,U_N\in\mbR^{d}$ be independently identically distributed according to $p_U$. Define $\tilde{S}_i, \tilde{Z}, \tilde{I}$ as 
	\begin{align}
		\tilde{S}_i:=\frac{1}{N} \left(\nabla_{\theta} \log\frac{1}{q_U(U_i;\theta) }
		- \nabla_{\theta}  \mbE_{p_{U}}
		\left[ \log \frac{1}{q_{U}(U;\theta)}\right]  \right),
	\end{align}
with $ i=1,\cdots,N$, and $\tilde{Z}:= \sum_{i=1}^N \tilde{S}_i, 
\tilde{I}:= \|\tilde{Z}\|^2, $ where $q_U(U;\theta):\mbR^{d}\times\mbR^m\to[0,1]$ is a function of $U\in\mbR^{d}$ and $\theta\in\mbR^m$. We assume that $q_U(U;\theta)$ is $\tilde{L}$-Lipschitz continuous with respect to $U$, i.e., 
	\begin{align} 
		\|q_{U}(U'; \theta) - q_{U}(U; \theta)\| 
		\leq \tilde{L} \left\| 
		U'
		- U\right\|, \forall U, U'\in\mbR^{d}, 
	\end{align}
	and that $q_{U}(U;\theta)$ does not vanish, i.e., 
	$
		q_{U}(U;\theta) \geq \tilde{q}_0, \forall  U\ \in\mbR^{n+1},
	$
	where $\tilde{L}>0$ and $\tilde{q}_0>0$ are constants. Then for any $\theta\in\mbR^m$, if $N\geq \left( \frac{2\tilde{L}^2}{\tilde{q}_0^2 t} + \frac{4\tilde{L}}{3 \tilde{q}_0 \sqrt{t}} \right)\log\frac{m+1}{\delta}$, we have
	$
		\mbP(\|\tilde{Z}\|^2 \leq t) \geq 1-\delta,
	$
	where $t, \delta>0$ are arbitrary constants.
\end{lemma}

Lemma \ref{Lem:Conineq_Log-Prob-Grad} shows that when the $\sum_{i=1}^N\log\frac{1}{q_U(U_i;\theta)}$ is used as the loss function associated with an example $U_i$, if the number of examples $N$ used to calculate the gradient estimate of the loss function with parameter $\theta$ is large enough,  then the gradient estimate will be very close to the truth gradient calculated from the whole data distribution $p_U$. The is consistent with our intuitions, i.e., in the extreme case where all the examples from the distribution are used, the gradient estimate will be the same as the one calculate from the the whole distribution. The sample complexity is also intuitive, e.g., when $\tilde{L}$ is small (the function $q_U(U;\theta)$ varies mildly), the $N$ can be small (we need less examples to accurately estimate the gradient). Lemma \ref{Lem:Conineq_Log-Prob-Grad} can be applied to show that $ \mbE\left[\|\nabla \mcL^{(N)}(\theta_t) - \nabla \mcL(\theta_t) \|^2 | \theta_t\right]$ can be very small when $N$ is large enough. The result is formally presented in Theorem \ref{Thm:GradConcentrationIneq}.

\begin{thm}\label{Thm:GradConcentrationIneq}
	(Concentration Inequality for Mutual Information Learning Loss Function Gradient) We consider the $\|\nabla\mcL^{(N)}(\theta_t) - \nabla \mcL(\theta_t) \|^2$ in each iteration of the Algorithm \ref{Alg:SGD} in Theorem \ref{Thm:cnvgcRt} where $\theta_t$ contains $\theta_{Y|X}\in\mbR^m$ and $\theta_Y\in\mbR^{m'}$. We assume that $q_{Y|X}(Y|X;\theta_{Y|X})$ is $\tilde{L}$-Lipschitz continuous with respect to $(X,Y)$, i.e., $\forall (X,Y), (X',Y')\in\mbR^{m}\times\mbR$, 
	\begin{align*} 
		\|q_{Y|X}(Y'|X'; \theta_{Y|X}) - q_{Y|X}(Y|X; \theta_{Y|X})\| 
		\leq \tilde{L} \left\| 
		\left[\begin{matrix}
			X' \\ Y'
		\end{matrix}\right]
		- 
		\left[\begin{matrix}
			X \\ Y
		\end{matrix}\right]
		\right\|, 
	\end{align*}
	and that $q_{Y|X}(Y|X;\theta_{Y|X})$ does not vanish, i.e., 
	\begin{align} 
		q_{Y|X}(Y|X;\theta_{Y|X}) \geq \tilde{q}_0, \forall  (X,Y)\in\mbR^{m}\times\mbR,
	\end{align}  
	where $\tilde{L}>0$ and $\tilde{q}_0>0$ are constants. We also assume that $q_{Y}(Y;\theta_{Y})$ is $\bar{L}$-Lipschitz continuous with respect to $Y$, i.e., 
	\begin{align*} 
		\|q_{Y}(Y'; \theta_Y) - q_{Y}(Y; \theta_Y)\| 
		\leq \bar{L} \left\| 
		Y' - Y
		\right\|, \forall Y,Y'\in\mbR, 
	\end{align*}
	and that $q_{Y}(Y;\theta_{Y})$ does not vanish, i.e., 
	\begin{align} 
		q_{Y}(Y;\theta_Y) \geq \bar{q}_0, \forall  Y\in\mbR,
	\end{align}  
	where $\bar{L}>0$ and $\bar{q}_0>0$ are constants. For any $\epsilon>0, \delta>0$, if $N\geq \max\left(N_1,N_2\right)$ where
	\begin{align} 
		 N_1:=\left(\frac{4\tilde{L}^2}{\tilde{q}_0^2\epsilon} + \frac{4\sqrt{2}\tilde{L}}{3\tilde{q}_0\sqrt{\epsilon}} \right)\log\frac{2(m+1)}{\delta} , \\
		 N_2:=	\left( \frac{4\bar{L}^2\lambda_{ent}^2}{\bar{q}_0^2\epsilon} + \frac{4\sqrt{2}\bar{L}\lambda_{ent}}{3\bar{q}_0 \sqrt{\epsilon}} \right) \log\frac{2(m'+1)}{\delta},
	\end{align}
we have
$
		\mbP_{\mcS}\left( \|\nabla\mcL^{(N)}(\theta_t) - \nabla \mcL(\theta_t) \|^2 \leq \epsilon \right) \geq 1-\delta
$
where the expectation is over $\mcS=\{(X_i,Y_i)\}_{i=1}^N$.
\end{thm}

Theorem \ref{Thm:GradConcentrationIneq} shows that in each iteration of the Algorithm \ref{Alg:SGD}, conditioning on $\theta_t$ in previous iteration, the empirical gradient $\nabla \mcL^{(N)}(\theta_t)$ from empirical sample $\mcS$ can be very close to the true gradient $\nabla\mcL(\theta_t)$ from the distribution when enough examples are sampled in each iteration. This is also intuitive since more examples will be more representative of the data distribution. This results in that the $\mbE[\|\nabla\mcL^{(N)}(\theta_t) - \nabla \mcL(\theta_t) \|^2]$ in \eqref{Eq:IterationComplexity} can be very small which allows faster convergence speed. See more details in the Supplemental Materials. We want to mention that Theorem \ref{Thm:GradConcentrationIneq} only holds nonuniformly for $\theta_t$, and the uniform concentration result for all $\theta_t,t=0,\cdots,T-1$ is challenging due to the searching space $\mbR^{m+m'}$. The chaining method can be a good option for handling it, and we leave it for future work \cite{yi_mutual_2022,asadi_chaining_2018}.

\section{Generalization Performance Bound}

In this section, we consider the generalization performance of arbitrary regression models trained on sample from $p_{X,Y}\in\mbR^n\times \mbR$, and show how it is related to the mutual information associated with the data distribution for a Gaussian data model. We follow Yi et al. \cite{yi_trust_2019,yi_derivation_2020,yi_towards_2021,yi_mutual_2022} where they considered the classification tasks under a encoding-decoding paradigm, and we consider the regression problem also in the same setting. We model the regression task as $Y\to X \to \hat{Y}$, where $\hat{Y}$ is an estimate of $Y$ obtained from $X$. This is consistent with the practice in many applications \cite{yang_unsupervised_2022,ren_balanced_2022,griffin_depth_2021}. For example, in an image depth estimation task commonly encountered in 3D vision where we want to estimate the distance from the scene to the camera, when the photographer takes pictures, he will first determine how far he should stand from the scene (regression label $Y$), and then take the pictures ($X$). The depth estimation task will then give an depth estimate $\hat{Y}$ \cite{griffin_depth_2021}. 

We will evaluate the generalization performance of the regression tasks via population loss which is defined as 
\begin{align*}
	\mcR:=E[(Y-\hat{Y})^2] = \int_{X,Y} p_{X,Y}(x,y) (y - \hat{y}(x))^2 dxdy,
\end{align*}
where each $\hat{y}$ is a function of $x$. Our generalization performance bound depends on Fano's inequality for continuous random variables \cite{cover_elements_2012}.

\subsection{Multiple-Output Regression Data Model}

In this section, we consider a data model over correlated joint Gaussian distribution $p_{X,Y}$, i.e., 
\begin{align}\label{Defn:RepresentationLearningDataModel}
	X=\rho Y + \sqrt{1-\rho^2} Z, \rho\in(0,1)
\end{align} 
where the elements of $Y\in\mbR^n$ follow I.I.D. standard Gaussian distribution $\mcN(0,I_n)$, and the elements of $Z\in\mbR^n$ follow I.I.D. standard Gaussian distribution $\mcN(0,I_n)$, and the $Y$ and $Z$ are independent. The \eqref{Defn:RepresentationLearningDataModel} can be treated as a simplified model underlying many applications \cite{he_deliberation_2019,zhu_unpaired_2017,yang_unsupervised_2022,preechakul_diffusion_2022,amjad_learning_2020}. An example is the image translation tasks where $Y$ is an image in a desired domain such as an art image of Monet style, and $X$ is a natural style image generated from $Y$ via a certain transformation $g$ \cite{zhu_unpaired_2017}. The goal in this particular scenario is to synthesize Monet style images from nature images which can be beneficial for arts creation \cite{zhu_unpaired_2017}. In \eqref{Defn:RepresentationLearningDataModel}, we simplify the $g$ as the composition of a scaling operation and an additive noise perturbation. The mutual information associated with the data generation distribution in \eqref{Defn:RepresentationLearningDataModel} can be derived, and the results are formally presented in Theorem \ref{Thm:MIGTinGaussian}.

\begin{thm}\label{Thm:MIGTinGaussian}
	(Mutual Information of Multi-output Regression Data Model) We consider a multi-output regression task where the input $Y\in\mbR^n$ of a machine learning systems has all its elements folllowing I.I.D. standard Gaussian distribution, and the output $X\in\mbR^n$ is generated according to \eqref{Defn:RepresentationLearningDataModel}. Then we have 
$
		I(X;Y) = \frac{n}{2} \log\frac{1}{1-\rho^2}.$
	
\end{thm}

Theorem \ref{Thm:MIGTinGaussian} implies that the mutual information $I(X;Y)$ can increase when the dimensionality $n$ and the scaling factor $\rho$ increases. This is consistent with our intuitions, e.g., a large $\rho$ implies the $X$ is more dependent on $Y$ and less dependent on the noise $Z$, thus a strong depency between $X,Y$. Based on Theorem \ref{Thm:MIGTinGaussian}, we can give the generalization loss associated with the data distribution for any models trained on the dataset, and the resuls are formally presented in Corollary \ref{Cor:GeneralizationLossViaMI}.

\begin{cor}\label{Cor:GeneralizationLossViaMI}
	(Generalization Loss Lower bound in via Mutual Information for \eqref{Defn:RepresentationLearningDataModel}) We consider a multi-output regression task where the input $Y\in\mbR^n$ of a machine learning systems has all its elements folllowing I.I.D. standard Gaussian distribution, and the output $X\in\mbR^n$ is generated according to \eqref{Defn:RepresentationLearningDataModel}. Then for any estimator $\hat{Y}$ from $X$, we have
	$
		\mcR \geq b(n,\rho):=(2\pi e)^{\frac{n-2}{2}} (1-\rho^2)^{\frac{n}{2}}.
	$
	Moreover,
	\begin{align}\label{Eq:Asymptotic}
		\lim_{n\to\infty} b(n,\rho) = 
		\begin{cases}
			\infty, \ \text{if } \rho\in\left(0,\sqrt{1-\frac{1}{2\pi e}}\right), \\
			\frac{1}{2\pi e},  \ \text{if } \rho = \sqrt{1-\frac{1}{2\pi e}}, \\
			0,  \ \text{if } \rho\in\left(\sqrt{1-\frac{1}{2\pi e}}, 1\right).
		\end{cases}
	\end{align}
\end{cor} 

Corollary \ref{Cor:GeneralizationLossViaMI} implies that the generalization loss lower bound (GLLB) decreases when the $\rho$ increases. This is intuitive since a large $\rho$ indicates strong dependency between $X$ and $Y$ which makes it easier to infer $Y$ from $X$, thus a small generalization loss can be acheived. Corollary \ref{Cor:GeneralizationLossViaMI} also tells us that the dimensionality of the regression problem (e.g., $n$) alone cannot determining the hardness of the learning task, and the dependency between the $X$ and the $Y$ (e.g., $\rho$) determines how the dimensionality can affect the learning. When the dependency is strong enough, e.g., $\rho\in\left(\sqrt{1-\frac{1}{2\pi e}}, 1\right)$, a high dimensionality (large $n$) can be beneficial to the learning, and the generalization loss lower bound $b(n,\rho)$ can converge to 0 as $n$ goes to infinity. However, if the dependency is weak, e.g., $\rho\in\left(0,\sqrt{1-\frac{1}{2\pi e}}\right)$, the high dimensionality can be a curse as the $b(n,\rho)$ can go to infinity when $n$ goes to infinity. This implies that the high dimensionality can be a bless instead of a curse in some scenarios \cite{donoho_high-dimensional_2000,gorban_high--dimensional_2020,gorban_blessing_2018}.

\section{Conclusions}\label{Sec:Conclusions}

This paper, we showed that the existing commonly used mean square error minimization approach for regression tasks is equivalent to a conditional differential entropy learning task. Motivated by the success of the mutual information learned classifiers (MILCs) in classification tasks, we extended it to a regression task and proposed the mutual information learned regressor (MILR) framework. The missing convergence analysis of SGD for training MILCs in \cite{yi_mutual_2022} motivates us to give a theoretical convergence of the SGD algorithm in regression tasks. Such convergence analysis can be applied for the classification tasks without much efforts. To better appreaciate the connection between the generalization performance of the regression models and the mutual information associated with the data distributon which is used to train them, we considered a multi-class regression data model, and derived the generalization performance lower bound in terms of the mutual information.

\bibliography{202204_MILR}

\onecolumn

\appendix

\section{Missing Proofs}

In this section, we present the detailed proof the results in the paper, i.e., Theorem 2.1 which is restated as in Theorem \ref{Thm:CndtDiffEtrpLnViaCrsDiffEtrpMin}, Theorem 2.2 which is restated as in Theorem \ref{Thm:DiffEtrpLnViaCrsEtrpMin}, Theorem 4.1 which is restated as in Theorem \ref{Thm:cnvgcRt}, Lemma 4.2 which is restated as in Lemma \ref{Lem:Conineq_Log-Prob-Grad}, Lemma 4.3 which is restated as in Lemma \ref{Thm:GradConcentrationIneq}, Theorem 5.1 which is restated as in Lemma \ref{Thm:MIGTinGaussian}, and Corollary 1 which is restated as in Corollary \ref{Cor:GeneralizationLossViaMI}.

For self-containedness, we restate some of the notations which will be used for later proofs. We consider the stochastic gradient descent (SGD) algorithm for solving an constrained problem of \eqref{Defn:Unconstrained}, i.e., 
\begin{align}\label{Defn:Unconstrained}
	\inf_{\theta_Y, \theta_{Y|X}} \mcL(\theta),
\end{align}
where
\begin{align}\label{Defn:DistributionalLoss}
	\mcL(\theta)
	:= \lambda_{ent}\mbE_{p_Y} \left[\log \frac{1}{q_Y(Y; \theta_Y)} \right] + \mbE_{p_{X,Y}} \left[\log\frac{1}{q_{Y|X}(Y|X; \theta_{Y|X})} \right], 
\end{align}
and $\theta:=[\theta_{Y|X}^T \theta_Y^T]^T\in\mbR^{m+m'}$ with $\theta_{X|Y}\in\Theta\subset\mbR^m$ and $\theta_Y\in\Gamma\in\mbR^{m'}$. The gradient of $\mcL(\theta)$ can be computed as
\begin{align}\label{Eq:gradOfmcL}
	\nabla \mcL(\theta)
	= \left[\begin{matrix}
		\nabla_{\theta_{Y|X}} \mcL(\theta) \\ 
		\nabla_{\theta_Y} \mcL(\theta)
	\end{matrix}\right] 
	= \left[\begin{matrix}
		\mbE_{p_{X,Y}}\left[ - \frac{\nabla_{\theta_{Y|X}} q_{Y|X}(Y|X;\theta_{Y|X})}{q_{Y|X}(Y|X;\theta_{Y|X})} \right] \\ 
		\lambda_{ent} \mbE_{p_Y} \left[ - \frac{\nabla_{\theta_Y} q_Y(Y;\theta_Y)}{q_Y(Y;\theta_Y)} \right]
	\end{matrix}\right] 
\end{align}
The SGD updating rules are presented in Algorithm \ref{Alg:SGD} where in each iteration $t$, we randomly sample a batch of $N$ data points $\{(x_i^t,y_i^t)\}_{i=1}^N$ with $(x_i,y_i)$ being I.I.D. according to $p_{X,Y}\in\mbR^n\times\mbR$. An estimate of the gradient $\nabla \mcL(\theta_t)$ at $\theta_t$ from the sample batch will be used to update the model parameters $\theta$, i.e., $\theta_{t+1}:=\theta_t - \eta_t  \nabla \mcL^{(N)}(\theta_t)$ where $\nabla \mcL^{(N)}(\theta_t)$ is defined in Algorithm \ref{Alg:SGD}. Starting from a given intialization $\theta_0\in\mbR^{m+m'}$, the process continues until it converges. In Theorem \ref{Thm:cnvgcRt}, we give a convergence analysis of Algorithm \ref{Alg:SGD}.

\begin{algorithm} 
	\caption{Stochastic Gradient Descent for Solving \eqref{Defn:Unconstrained}} 
	\label{Alg:SGD} 
	\begin{algorithmic} 
		\REQUIRE Learning rate $\{\eta_t\}_{t=0}^{T-1}$ 
		\ENSURE Initialization $\theta_0\in\mbR^{m+m'}$
		\FOR{$t=0,1,\cdots,T-1$}
		\STATE Sample $\mcS_{t+1}:=\{(X_i^{t+1}, Y_i^{t+1})\}_{i=1}^N$ with I.I.D. $(X_i^{t+1}, Y_i^{t+1}) \sim p_{X,Y}\in\mbR^{n}\times\mbR$
		\STATE $\nabla \mcL^{(N)}(\theta_t):=\frac{1}{N} \sum_{i=1}^N \nabla	\ell(\theta_t; X_i^{t+1},Y_i^{t+1})$ with 
		\begin{align}
			\ell(\theta_t; X_i^{t+1},Y_i^{t+1})
			& :=\lambda_{ent} \log\frac{1}{q_Y(Y_i^{t+1}; \theta_Y^{t})}  \nonumber \\
			&\quad + \log\frac{1}{q_{Y|X}(Y_i^{t+1}|X_i^{t+1}; \theta_{Y|X}^t)}
		\end{align}
		\STATE $\theta_{t+1} = \theta_t - \eta_t \nabla \mcL^{(N)}$
		\ENDFOR
	\end{algorithmic}
\end{algorithm}

We first present the proof of Theorem 2.2 which is restated as in Theorem \ref{Thm:DiffEtrpLnViaCrsEtrpMin}.

\begin{thm}\label{Thm:CndtDiffEtrpLnViaCrsDiffEtrpMin}
	(Conditional Differential Entropy Learning via Conditional Cross Entropy Minimization) For an arbitrary joint distribution $p_{X,Y}$ of two continuous random variables or vectors $X$ and $Y$, we have
	\begin{align}
		h(Y|X)\leq \inf_{q_{Y|X}} H(p_{Y|X}, q_{Y|X}),
	\end{align}
	where the conditional cross differential entropy is defined as
	\begin{align}\label{Eq:CndtCrsDiffEntrp}
		H(p_{Y|X}, q_{Y|X}):=\int_{x,y} p_{X,Y}(x,y) \log\left( \frac{1}{q_{Y|X}(y|x)} \right)dxdy.
	\end{align} 
	The equality holds if and only $q_{Y|X}=p_{Y|X}$. Moreover, let $\hat{P}_{Y|X}$ be an empirical conditional distribution of $Y$ and $\hat{P}_X$ be an empirical distribution of $X$, and define 
	\begin{align}
		R_{Y|X}:=\frac{p_{Y|X}}{\hat{p}_{Y|X}}, R_X:= \frac{p_X}{\hat{p}_X}.
	\end{align}
	Then 
	\begin{align}\label{Eq:CndtCrsDiffEntrp_Emprc}
		h(Y|X)\leq \inf_{q^g_{Y|X}} H(p^g_{Y|X}, q^g_{Y|X}),
	\end{align}
	where 
	\begin{align}
		p^g_{X,Y}:=p^g_X p^g_{Y|X}, 
		p^g_X:=R_Xp_X, 
		p^g_{Y|X}:=R_{Y|X}p_{Y|X}, 
		q^g_{Y|X}:=q_{Y|X}/R_{Y|X}.
	\end{align}
	The equality holds if and only if $p_Y = \hat{p}_Y = q_Y$.
	
\end{thm}

\begin{proof}
	(of Theorem \ref{Thm:CndtDiffEtrpLnViaCrsDiffEtrpMin}) We follow the ideas from \cite{yi_mutual_2022} to prove the theorem. From the definition of conditional differential entropy, we have
	\begin{align}
		h(Y|X)
		& = \int_{x,y} p_{X,Y}(x,y) \log\left( \frac{1}{p_{Y|X}(y|x)} \right)dxdy \nonumber \\
		& = \int_{x,y} p_{X,Y}(x,y) \log\left( \frac{q_{Y|X}(y|x)}{p_{Y|X}(y|x)} \right)dxdy + \int_{x,y} p_{X,Y}(x,y) \log\left( \frac{1}{q_{Y|X}(y|x)} \right)dxdy \nonumber\\
		& \leq H(p_{Y|X}, q_{Y|X}) - D_{KL}(p_{Y|X}||q_{Y|X})\nonumber\\
		& \leq H(p_{Y|X}, q_{Y|X}),
	\end{align}
	where we used \eqref{Eq:CndtCrsDiffEntrp} and the fact that the KL divergence $D_{KL}(p_{Y|X}||q_{Y|X})\geq0$. The equality holds if and only if $q_{Y|X}=p_{Y|X}$. Since the above inequality holds for all $q_{Y|X}$, we have
	\begin{align}
		h(Y|X)\leq \inf_{q_{Y|X}}H(p_{Y|X}, q_{Y|X}).
	\end{align}
	
	Since 
	\begin{align}
		H(p_{Y|X}, q_{Y|X})
		& = \int_{x,y} p_{X,Y}(x,y) \log\left( \frac{1}{q_{Y|X}(x|y)} \frac{\hat{p}_{Y|X}(y|x)} {{p_{Y|X}(y|x)}} \frac{{p}_{Y|X}(y|x)} {{\hat{p}_{Y|X}(y|x)}}  \right)dxdy \\
		& = \int_{x,y} p_{X,Y}(x,y) \log\left(  \frac{\hat{p}_{Y|X}(y|x)} {{p_{Y|X}(y|x)}}   \right)dxdy  + \int_{x,y} p_{X,Y}(x,y) \log\left( \frac{1}{q_{Y|X}(x|y)} \frac{{p}_{Y|X}(y|x)} {{\hat{p}_{Y|X}(y|x)}}  \right)dxdy \\
		&=  \int_{x,y} p_{X,Y}(x,y) \log\left( \frac{1}{q_{Y|X}(x|y)} \frac{{p}_{Y|X}(y|x)} {{\hat{p}_{Y|X}(y|x)}}  \right)dxdy - D_{KL}(p_{Y|X}||\hat{p}_{Y|X}) \\
		& \leq  \int_{x,y} p_{X,Y}(x,y) \log\left( \frac{1}{q_{Y|X}(x|y)} \frac{{p}_{Y|X}(y|x)} {{\hat{p}_{Y|X}(y|x)}}  \right)dxdy \\
		& = \int_{x,y} \frac{{p}_{X}(x)}{\hat{p}_{X}(x)} \hat{p}_{X}(x) \frac{p_{Y|X}(y|x)}{\hat{p}_{Y|X}(y|x)} \hat{p}_{Y|X}(y|x)  \log\left( \frac{1}{q_{Y|X}(x|y)} \frac{{p}_{Y|X}(y|x)} {{\hat{p}_{Y|X}(y|x)}}  \right)dxdy\\
		& = \int_{x,y} R_{X}(x) \hat{p}_{X}(x) R_{Y|X}(y|x) \hat{p}_{Y|X}(y|x)  \log\left( \frac{1}{q_{Y|X}(x|y)} R_{Y|X}(y|x)  \right)dxdy \\
		& = \int_{x,y} p^g_X(x)p^g_{Y|X}(y|x)\log\left( \frac{1}{q^g_{Y|X}(y|x)} \right) dxdy \\
		& = H(p_{Y|X}^g||q^g_{Y|X}).
	\end{align}
	The equality holds if and only if $p_X=\hat{p}_X$ and $p_{Y|X}=\hat{p}_{Y|X}$. Thus, 
	\begin{align}
		h(Y|X) \leq H(p_{Y|X}^g||q^g_{Y|X}),
	\end{align}
	and the equality holds if and only if $p_X=\hat{p}_X$ and $p_{Y|X}=\hat{p}_{Y|X}=q_{Y|X}$.
\end{proof}

We now give the proof of Theorem 2.1 as restated in Theorem \ref{Thm:DiffEtrpLnViaCrsEtrpMin}.

\begin{thm}\label{Thm:DiffEtrpLnViaCrsEtrpMin}
	(Differential Entropy Learning Via Cross Entropy Minimization) For a continuous random variable $Y\sim p_Y(Y)$, we have
	\begin{align}
		h(Y)\leq \inf_{q_Y(Y)} H(p_Y(Y), q_Y(Y)),
	\end{align} 
	where $q_Y(Y)$ is a valid distribution of $Y$, and the $H(p_Y(Y), q_Y(Y))$ is the cross entropy. The equality holds if and only if $q_Y = p_Y$. Let $\hat{p}_Y(Y)$ be an empirical distribution of $Y$, and define $R_Y(y) = \frac{p_Y(y)}{\hat{p}_Y(y)}, \forall y$. Then 
	\begin{align}
		h(Y) \leq \inf_{q^g_Y(Y)} H(p^g_Y(Y), q^g_Y(Y)),
	\end{align}
	where 
	\begin{align}
		q^g_Y(y) := \hat{p}_Y(y)R_Y(y), p^g_Y(y) := q_Y(y)/R_Y(y), \forall y.
	\end{align}
	The equality holds if and only if $p_Y = \hat{p}_Y = q_Y$.
\end{thm}

\begin{proof}
	(of Theorem \ref{Thm:DiffEtrpLnViaCrsEtrpMin}) We follow the ideas from [Yi et al., 2022] to prove Theorem \ref{Thm:DiffEtrpLnViaCrsEtrpMin}. For an arbitrary distribution $q_Y$ of $Y$, we have
	\begin{align}\label{Eq:lblDiffEntrp}
		h(Y) 
		& = \int_y p_Y(y) \log\left( \frac{1}{p_Y(y)} \right)dy \\
		& = \int_y p_Y(y) \log\left( \frac{q_Y(y)}{p_Y(y)} \right)dy + \int_y p_Y(y) \log\left( \frac{1}{q_Y(y)} \right)dy \\
		& = -D(p_Y||q_Y) + H(p_Y, q_Y) \\
		& \leq H(p_Y, q_Y)
	\end{align}
	where we used the fact that the KL divergence $D(p_Y||q_Y)\geq 0$. Since \eqref{Eq:lblDiffEntrp} holds for all $q_Y$, then we have
	\begin{align}
		h(Y) \leq \inf_{q_Y} H(p_Y, q_Y).
	\end{align}
	The equality holds if and only if $D(p_Y||q_Y)=0$, i.e., $p_Y=q_Y$. 
	
	For arbitrary empirical distribution $\hat{p}_Y$, we have
	\begin{align}
		H(p_Y, q_Y) 
		& = \int_y p_Y(y) \log\left( \frac{1}{q_Y(y)} \right)dy\\
		& =  \int_y p_Y(y) \log\left( \frac{\hat{p}_Y(y)}{p_y(y)}  \right)dy + \int_y p_Y(y) \log\left( \frac{1}{q_Y(y)}  \frac{p_Y(y)}{\hat{p}_Y(y)} \right)dy \\
		& = - D(p_Y||\hat{p}_Y) + \int_y \hat{p}_Y(y) \frac{p_Y(y)}{\hat{p}_Y(y)} \log\left( \frac{1}{q_Y(y)}  \frac{p_Y(y)}{\hat{p}_Y(y)} \right)dy \\
		& \leq \int_y p^g_Y(y) \log\left( \frac{1}{q^g_Y(y)} \right) dy\\
		& = H(p^g_Y||q^g_Y),
	\end{align}
	and the equality holds if and only if $D(p_Y||q_Y)=0$, i.e., $p_Y=\hat{p}_Y$. Thus, 
	\begin{align}
		h(Y) \leq \inf_{q^g_Y} H(p^g_Y(Y), q^g_Y(Y)),
	\end{align}
	and the equality holds if and only if $p_Y = \hat{p}_Y = q_Y$.
\end{proof}

We now present the proof of Theorem 4.1 which is restated as in Theorem \ref{Thm:cnvgcRt}.

\begin{thm}\label{Thm:cnvgcRt}
	(Convergence Guarantees of Stochastic Gradient Descent for Solving \eqref{Defn:Unconstrained}) We consider the problem defined in  \eqref{Defn:Unconstrained}, and assume the $\mcL(\theta)$ is $s$-smooth, i.e., for a constant $s>0$, 
	\begin{align}\label{Defn:s-smooth}
		\|\nabla \mcL(\theta) - \nabla \mcL(\theta')\| \leq s\|\theta - \theta'\|, \forall \theta, \theta'\in\mbR^{m+m'}.
	\end{align}
	At each iteration of Algorithm \ref{Alg:SGD} for solving \eqref{Defn:Unconstrained}, let the step size $\eta_t\in\left(0,\frac{2}{s}\right)$. Define $\Delta(\theta_0):= \mcL(\theta_0) - \inf_\theta \mcL(\theta)$ where $\theta_0\in\mbR^{m+m'}$ is an initialization. Then, if 
	\begin{align} \label{Eq:IterationComplexity}
		T\geq \frac{\Delta(\theta_0)}{\alpha \epsilon} + \frac{2}{s\alpha \epsilon} \sum_{t=0}^{T-1} \mbE[\|\nabla\mcL^{(N)}(\theta_t) - \nabla \mcL(\theta_t) \|^2]
	\end{align} 
	where $\alpha:=\min_{i=0,1,\cdots,T-1} -(\frac{1}{2} s \eta_t^2 - \eta_t)$, the expectation is with respect to $\theta_t, S_{t+1}, t=1,\dots,N$, we have $\frac{1}{T} \sum_{t=1}^{T-1} \mbE\left[|| \nabla_{\theta} \mcL(\theta_t)\|\right] \leq \epsilon$ where $\epsilon>0$ is a constant and the expectation is with respect to $\theta_t, t=1,\cdots,N$.
\end{thm}

\begin{proof}
	(of Theorem \ref{Thm:cnvgcRt}) From the smoothness of $\mcL(\theta)$, we have
	\begin{align}
		\mcL (\theta_{t+1}) 
		& = \mcL(\theta_t) + \nabla \mcL(\theta_t)^T(\theta_{t+1} - \theta_t) + \frac{1}{2} (\theta_{t+1} - \theta_t)^T \nabla^2 F(\theta')(\theta_{t+1} - \theta_t)  \nonumber\\
		& \leq \mcL(\theta_t) + \nabla \mcL(\theta_t)^T(\theta_{t+1} - \theta_t) + \frac{s}{2} (\theta_{t+1} - \theta_t)^T (\theta_{t+1} - \theta_t)  \nonumber\\
		& = \mcL(\theta_t) - \eta_t \nabla \mcL(\theta_t)^T\nabla \mcL^{(N)}(\theta_t) + \frac{s}{2} \eta_t^2 \|\nabla \mcL^{(N)}(\theta)\|^2  \nonumber\\
		& = \mcL(\theta_t) - \eta_t \nabla \mcL(\theta_t)^T\nabla \mcL^{(N)}(\theta_t)  \nonumber\\
		&\quad + \frac{s}{2} \eta_t^2 \|\nabla \mcL^{(N)}(\theta_t) - \mcL(\theta_t) \|^2 - \frac{s}{2}\eta_t^2 \|\nabla \mcL(\theta_t)\|^2 + {s}\eta_t^2 \nabla\mcL(\theta_t)^T \mcL^{(N)}(\theta_t)  \nonumber\\
		&  = \mcL(\theta_t) + \frac{s}{2} \eta_t^2 \|\nabla \mcL^{(N)}(\theta_t) - \mcL(\theta_t) \|^2 - \frac{s}{2}\eta_t^2 \|\nabla \mcL(\theta_t)\|^2 + \left({s}\eta_t^2  - \eta_t\right) \nabla \mcL(\theta_t)^T\nabla \mcL^{(N)}(\theta_t)   \nonumber\\
		& =  \mcL(\theta_t) + \frac{s}{2} \eta_t^2 \|\nabla \mcL^{(N)}(\theta_t) - \mcL(\theta_t) \|^2 - \frac{s}{2}\eta_t^2 \|\nabla \mcL(\theta_t)\|^2  \nonumber\\
		&\quad + \left({s}\eta_t^2  - \eta_t\right) \nabla \mcL(\theta_t)^T( \nabla \mcL^{(N)}(\theta_t) - \nabla \mcL(\theta_t)) +  \left({s}\eta_t^2  - \eta_t\right) \|\nabla \mcL(\theta_t)\|^2  \nonumber\\
		& = \mcL(\theta_t) + \frac{s}{2}\eta_t^2 \|\nabla \mcL^{(N)}(\theta_t) - \nabla \mcL(\theta_t) \|^2 + (\frac{s}{2}\eta_t^2 - \eta_t)\|\nabla \mcL(\theta_t)\|^2 \nonumber\\
		&\quad + \left(s\eta_t^2 - \eta_t\right) \nabla \mcL(\theta_t)^T \left(\nabla \mcL^{(N)}(\theta_t) - \nabla \mcL(\theta_t)\right),
	\end{align}
	where $\theta':=\theta_t + \zeta (\theta_{t+1} - \theta_t)$ with $\zeta\in[0,1]$. Then, 
	\begin{align}
		- \left(\frac{s}{2}\eta_t^2 - \eta_t\right)\|\nabla \mcL(\theta_t)\|^2  
		& \leq \mcL(\theta_t) - \mcL(\theta_{t+1})  + \frac{s}{2}\eta_t^2 \|\nabla \mcL^{(N)}(\theta_t) - \nabla \mcL(\theta_t) \|^2 \nonumber \\
		&\quad + \left(s\eta_t^2 - \eta_t\right) \nabla \mcL(\theta_t)^T (\nabla \mcL^{(N)}(\theta_t) - \nabla \mcL(\theta_t))
	\end{align}
	
	
	By taking expectation over $\mcS_{t+1}:=\{(X_i^{t+1},Y_i^{t+1})\}_{i=1}^N$ conditioning on $\theta_t$, we have
	\begin{align}\label{Eq:ConditionalExpectation}
		- \left(\frac{s}{2}\eta_t^2 - \eta_t \right)\|\nabla \mcL(\theta_t)\|^2  
		\leq \mbE[\mcL(\theta_t) - \mcL(\theta_{t+1}) | \theta_t] + \frac{s}{2}\eta_t^2 \mbE\left[\|\nabla \mcL^{(N)}(\theta_t) - \nabla \mcL(\theta_t) \|^2 | \theta_t\right],
	\end{align}
	where we used the fact that
	\begin{align}\label{Eq:GradZeroGap}
		\mbE_{p_{S_{t+1}}}[\nabla \mcL^{(N)}(\theta_t) - \nabla \mcL(\theta_t)|\theta_t] = \bm{0},
	\end{align}
	where 
	\begin{align}
		p_{S_{t+1}}:= \prod_{i=1}^N p_{X_i^{t+1}, Y_i^{t+1}}, 
	\end{align}
	and $p_{X_i^{t+1}, Y_i^{t+1}}=p_{X,Y}, \forall i=1,\dots,N$ due to the I.I.D. assumption.
	
	To see \eqref{Eq:GradZeroGap}, we first get $\nabla_\theta \mcL^{(N)}(\theta_t)$ from Algorithm \ref{Alg:SGD} as 
	\begin{align}\label{Eq:gradOfmcLEmprc}
		\nabla_\theta \mcL^{(N)}(\theta_t)
		= \left[ \begin{matrix}
			\frac{1}{N} \sum_{i=1}^N \nabla_{\theta_{Y|X}}	\ell(\theta_t; X_i^{t+1},Y_i^{t+1}) \\
			\frac{1}{N} \sum_{i=1}^N \nabla_{\theta_Y} 	\ell(\theta_t; X_i^{t+1},Y_i^{t+1})
		\end{matrix} \right] 
		= \left[ \begin{matrix}
			\frac{1}{N} \sum_{i=1}^N - \frac{\nabla_{\theta_{Y|X}}q_{Y|X}(Y_i^{t+1}|X_i^{t+1};\theta_{Y|X}^t) }{q_{Y|X}(Y_i^{t+1}|X_i^{t+1};\theta_{Y|X}^t)} \\
			\frac{1}{N} \sum_{i=1}^N - \lambda_{ent} \frac{\nabla_{\theta_Y} q_Y(Y_i; \theta_Y^t)}{ q_Y(Y_i; \theta_Y^t)}
		\end{matrix} \right] \in\mbR^{m+m'}.
	\end{align}
	Then, we have from \eqref{Eq:gradOfmcLEmprc} and \eqref{Eq:gradOfmcL}
	\begin{align}
		& \mbE_{p_{S_{t+1}}}[\nabla \mcL^{(N)}(\theta_t) - \nabla \mcL(\theta_t)|\theta_t] \\
		& = \mbE_{p_{S_t}}\left[ 
		\left[ \begin{matrix}
			\frac{1}{N} \sum_{i=1}^N - \frac{\nabla_{\theta_{Y|X}}q_{Y|X}(Y_i^{t+1}|X_i^{t+1};\theta_{Y|X}^t) }{q_{Y|X}(Y_i^{t+1}|X_i^{t+1};\theta_{Y|X}^t)} \\
			\frac{1}{N} \sum_{i=1}^N - \lambda_{ent} \frac{\nabla_{\theta_Y} q_Y(Y_i; \theta_Y^t)}{ q_Y(Y_i; \theta_Y^t)}
		\end{matrix} \right]
		- 
		\left[\begin{matrix}
			\mbE_{p_{X,Y}}\left[ - \frac{\nabla_{\theta_{Y|X}} q_{Y|X}(Y|X;\theta_{Y|X})}{q_{Y|X}(Y|X;\theta_{Y|X}^t)} \right] \\ 
			\lambda_{ent} \mbE_{p_Y} \left[ - \frac{\nabla_{\theta_Y} q_Y(Y;\theta_Y)}{q_Y(Y;\theta_Y^t)} \right]
		\end{matrix}\right]
		\bigg| \theta_t\right] \\
		& =\left[ 
		\left[ \begin{matrix}
			\frac{1}{N} \sum_{i=1}^N \mbE_{p_{X_i^{t+1}, Y_i^{t+1}}} \left[- \frac{\nabla_{\theta_{Y|X}}q_{Y|X}(Y_i^{t+1}|X_i^{t+1};\theta_{Y|X}^t) }{q_{Y|X}(Y_i^{t+1}|X_i^{t+1};\theta_{Y|X}^t)} \right] \\
			\frac{1}{N} \sum_{i=1}^N \mbE_{p_{X_i^{t+1}, Y_i^{t+1}}} \left[- \lambda_{ent} \frac{\nabla_{\theta_Y} q_Y(Y_i; \theta_Y)}{ q_Y(Y_i; \theta_Y^t)}\right]
		\end{matrix} \right]\right]
		- 
		\left[\begin{matrix}
			\mbE_{p_{X,Y}}\left[ - \frac{\nabla_{\theta_{Y|X}} q_{Y|X}(Y|X;\theta_{Y|X}^t)}{q_{Y|X}(Y|X;\theta_{Y|X}^t)} \right] \\ 
			\lambda_{ent} \mbE_{p_Y} \left[ - \frac{\nabla_{\theta_Y} q_Y(Y;\theta_Y^t)}{q_Y(Y;\theta_Y^t)} \right]
		\end{matrix}\right] \\
		& =\left[ 
		\left[ \begin{matrix}
			\frac{1}{N} \sum_{i=1}^N \mbE_{p_{X, Y}} \left[- \frac{\nabla_{\theta_{Y|X}}q_{Y|X}(Y|X;\theta_{Y|X}^t) }{q_{Y|X}(Y|X;\theta_{Y|X}^t)} \right] \\
			\frac{1}{N} \sum_{i=1}^N \mbE_{p_{X, Y}} \left[- \lambda_{ent} \frac{\nabla_{\theta_Y} q_Y(Y; \theta_Y^t)}{ q_Y(Y; \theta_Y^t)} \right]
		\end{matrix} \right]\right]
		- 
		\left[\begin{matrix}
			\mbE_{p_{X,Y}}\left[ - \frac{\nabla_{\theta_{Y|X}} q_{Y|X}(Y|X;\theta_{Y|X})}{q_{Y|X}(Y|X;\theta_{Y|X})} \right] \\ 
			\lambda_{ent} \mbE_{p_Y} \left[ - \frac{\nabla_{\theta_Y} q_Y(Y;\theta_Y)}{q_Y(Y;\theta_Y)} \right]
		\end{matrix}\right] = \bm{0}.
	\end{align}
	
	By taking total expectation over both $\theta_t$ and $\mcS_{t+1}$, we have
	\begin{align}
		- \left(\frac{s}{2}\eta_t^2 - \eta_t \right)\mbE[\|\nabla \mcL(\theta_t)\|^2]  
		& \leq \mbE[\mcL(\theta_t) - \mcL(\theta_{t+1})] + \frac{s}{2}\eta_t^2 \mbE\left[\|\nabla \mcL^{(N)}(\theta_t) - \nabla \mcL(\theta_t) \|^2 \right]
	\end{align}
	Summing over $t=0,1,\cdots, T-1$, we have
	\begin{align}
		\alpha \sum_{t=0}^{T-1} \mbE[\|\nabla \mcL(\theta_t)\|^2] 
		& \leq \mbE[\mcL(\theta_0) - \mcL(\theta_{T})] + \frac{s}{2}\eta_{max}^2 \sum_{t=0}^{T-1} \mbE[\|\nabla\mcL^{(N)}(\theta_t) - \nabla \mcL(\theta_t) \|^2] \\
		& \leq \mbE[\mcL(\theta_0) - \inf_\theta \mcL(\theta)] + \frac{s}{2}\eta_{max}^2 \sum_{t=0}^{T-1} \mbE[\|\nabla\mcL^{(N)}(\theta_t) - \nabla \mcL(\theta_t) \|^2] \\
		& = \mcL(\theta_0) - \inf_\theta \mcL(\theta) + \frac{s}{2}\eta_{max}^2 \sum_{t=0}^{T-1} \mbE[\|\nabla\mcL^{(N)}(\theta_t) - \nabla \mcL(\theta_t) \|^2]\\
		& =\Delta(\theta_0) + \frac{s}{2}\eta_{max}^2 \sum_{t=0}^{T-1} \mbE[\|\nabla\mcL^{(N)}(\theta_t) - \nabla \mcL(\theta_t) \|^2]
	\end{align}
	where we define $\alpha:=\min_{i=0,1,\cdots,T-1} -(\frac{1}{2} s \eta_t^2 - \eta_t)$, $\eta_{max}:=\max_{i=0,1,\cdots,T-1} \eta_t$. Thus, 
	\begin{align}
		\frac{1}{T} \sum_{t=0}^{T-1} \mbE[\|\nabla \mcL(\theta_t)\|^2]  
		& \leq \frac{\Delta(\theta_0)}{\alpha T} + \frac{2}{s\alpha T}\sum_{t=0}^{T-1} \mbE[\|\nabla\mcL^{(N)}(\theta_t) - \nabla \mcL(\theta_t) \|^2],
	\end{align} 
	where we used the assumption that $\eta_t \in\left(0, \frac{2}{s}\right)$ and the fact that $\alpha>0$.
	When 
	\begin{align} 
		T\geq \frac{\Delta(\theta_0)}{\alpha \epsilon} + \frac{2}{s\alpha \epsilon} \sum_{t=0}^{T-1} \mbE[\|\nabla\mcL^{(N)}(\theta_t) - \nabla \mcL(\theta_t) \|^2], 
	\end{align} 
	we have
	\begin{align}
		\frac{1}{T} \sum_{t=0}^{T-1} \mbE[\|\nabla \mcL(\theta_t)\|^2]   \leq \epsilon,
	\end{align}
	where $\epsilon>0$ is a constant. 
\end{proof}

Before presenting the proof of Lemma \ref{Lem:Conineq_Log-Prob-Grad}, we present a technical lemma which will be used for proving Lemma \ref{Lem:Conineq_Log-Prob-Grad}.

\begin{lemma}\label{Lem:MtrxBersteinConcIneq}
	(Theorem 1.6.2 in \cite{tropp_introduction_2015}) Let $S_1, \cdots, S_n$ be independent, centered random matrices with common dimension $d_1\times d_2$, and assume that each one is uniformly bounded 
	\begin{align}
		\mbE[S_k] = 0, \text{ and } \|S_k\| \leq L, \text{ for each } k=1,\cdots,n.
	\end{align}
	Introduce the sum 
	\begin{align}
		Z:=\sum_{k=1}^n S_k,
	\end{align}
	and let $v(Z)$ denote the matrix variance statistic of the sum:
	\begin{align}
		v(Z):=\max\left\{ \|\mbE[ZZ^*]\|, \|\mbE[Z^*Z] \| \right\}
		= \max\left\{ \left\|\sum_{k=1}^n \mbE[S_kS_k^*] \right\|, \left\|\sum_{k=1}^n \mbE[S_k^*S_k] \right\| \right\}.
	\end{align}
	Then
	\begin{align}
		\mbP\left( \|Z\|\geq t \right) \leq (d_1 +d_2)\exp\left( \frac{-t^2/2}{v(Z) + Lt/3} \right), \forall t\geq 0.
	\end{align}
	Furthermore, 
	\begin{align}
		\mbE[\|Z\|] \leq \sqrt{2v(Z)\log(d_1+d_2)} + \frac{1}{3}L\log(d_1+d_2).
	\end{align}
\end{lemma}

Lemma \ref{Lem:MtrxBersteinConcIneq} tells us that for a sequence of zero-mean random matrices $S_1,\cdots,S_n$ with bounded magnitude, the norm of their sum $Z$ will have high probability of being small, i.e., $\mbP(||Z||\leq t) \geq 1- (d_1+d_2)\exp\left( \frac{-t^2/3}{v(Z)+Lt/3} \right)$ where $t\geq0$. This is essentially what we expect for $\mbE\left[ \|\nabla \mcL^{(N)}(\theta_t) - \nabla \mcL(\theta_t)\|^2 \right]$. In Theorem \ref{Thm:GradConcentrationIneq}, we consider the gradient deviation conditioning on previous $\theta_t$ in \eqref{Eq:ConditionalExpectation} where the expectation is with respect to $\mcS_{t+1}:=\{(X_i,Y_i)\}_{i=1}^N$. For simplicity of presentation, we drop the iteration index $t$ in Theorem \ref{Thm:GradConcentrationIneq}. 

We now present the detailed proof of Lemma 4.2 which is restated as in Lemma \ref{Lem:Conineq_Log-Prob-Grad}.

\begin{lemma}\label{Lem:Conineq_Log-Prob-Grad}
	(Concentration Inequality of Logarithm-probability Loss Function Gradient) Let $U_1,\cdots,U_N\in\mbR^{n}$ be independently identically distributed according to $p_U$. Define $\tilde{S}_i, \tilde{Z}, \tilde{I}$ as 
	\begin{align}
		\tilde{S}_i:=\frac{1}{N} \left(\nabla_{\theta} \log\frac{1}{q_U(U_i;\theta) }
		- \nabla_{\theta}  \mbE_{p_{U}}
		\left[ \log \frac{1}{q_{U}(U;\theta)}\right]  \right), i=1,\cdots,N, 
		\tilde{Z}:= \sum_{i=1}^N \tilde{S}_i, 
		\tilde{I}:= \|\tilde{Z}\|^2, 
	\end{align}
	where $q_U(U;\theta):\mbR^{n}\times\mbR^m\to[0,1]$ is a function of $U\in\mbR^{n}$ and $\theta\in\mbR^m$. We assume that $q_U(U;\theta)$ is $\tilde{L}$-Lipschitz continuous with respect to $U$, i.e., 
	\begin{align} 
		\|q_{U}(U'; \theta) - q_{U}(U; \theta)\| 
		\leq \tilde{L} \left\| 
		U'
		- U\right\|, \forall U, U'\in\mbR^{n}, 
	\end{align}
	and that $q_{U}(U;\theta)$ does not vanish, i.e., 
	\begin{align} 
		q_{U}(U;\theta) \geq \tilde{q}_0, \forall  U\ \in\mbR^{n},
	\end{align}  
	where $\tilde{L}>0$ and $\tilde{q}_0>0$ are constants. Then for any $\theta\in\mbR^m$, if $N\geq \left( \frac{2\tilde{L}^2}{\tilde{q}_0^2 t} + \frac{4\tilde{L}}{3 \tilde{q}_0 \sqrt{t}} \right)\log\frac{m+1}{\delta}$, we have
	\begin{align}
		\mbP(\|\tilde{Z}\|^2 \leq t) \geq 1-\delta,
	\end{align}
	where $t, \delta>0$ are arbitrary constants.
\end{lemma}

\begin{proof}\label{Proof:Conineq_Log-Prob-Grad}
	(of Lemma \ref{Lem:Conineq_Log-Prob-Grad}) From the definitions, We know
	\begin{align}
		\tilde{S}_i=\frac{1}{N} \left(\mbE_{p_{U}}
		\left[ \frac{\nabla_{\theta} q_{U}(U;\theta)}{q_U(U;\theta)} \right] - \frac{\nabla_{\theta}q_U(U_i;\theta) }{q_U(U_i;\theta)} \right), i=1,\cdots,N.
	\end{align}
	It is obvious that $\mbE_{p_{U_i}} [\tilde{S}_i] = 0, \forall i=1,\cdots,N$ since $U_i$ are I.I.D. according to $p_{U}$.
	
	We also have 
	\begin{align}\label{Eq:BdnessTildeS}
		\|\tilde{S}_i\|
		& = \left\| \frac{1}{N} \left(\mbE_{p_{U}}
		\left[ \frac{\nabla_{\theta} q_U(U;\theta)}{q_U(U;\theta)} \right] - \frac{\nabla_{\theta}q_U(U_i;\theta) }{q_{U}(U_i;\theta)} \right) \right\|\\
		& \leq \frac{1}{N} \left( \left\| \mbE_{p_{U}}
		\left[ \frac{\nabla_{\theta} q_{U}(U;\theta)}{q_{U}(U;\theta)} \right]\right\| + \left\| \frac{\nabla_{\theta}q_{U}(U_i;\theta) }{q_{U}(U_i;\theta)} \right\| \right) \label{Eq:triangleIneq}\\
		& \leq \frac{1}{N} \left( \mbE_{p_{U}}
		\left[ \left\|  \frac{\nabla_{\theta} q_{U}(U;\theta)}{q_{U}(U;\theta)} \right\| \right] + \left\| \frac{\nabla_{\theta}q_{U}(U_i;\theta) }{q_{U}(U_i;\theta)} \right\| \right) \label{Eq:JensenIneq},
	\end{align}
	where \eqref{Eq:triangleIneq} is due to triangle inequality and \eqref{Eq:JensenIneq} is due to Jensen's inequality on convex norm function. Since $q_{U}(U; \theta)$ is Lipschitz continuous, i.e., 
	\begin{align} 
		\|q_{U}(U'; \theta) - q_{U}(U; \theta)\| 
		\leq \tilde{L} \left\| 
		U'
		- U\right\|, \forall U, U'\in\mbR^{d}, 
	\end{align}
	and 
	\begin{align} 
		q_{U}(U;\theta) \geq \tilde{q}_0>0, \forall  U\ \in\mbR^{n+1}
	\end{align}  
	we have
	\begin{align}\label{Eq:SmdUpBd} 
		\left\|  \frac{\nabla_{\theta} q_{U}(U;\theta)}{q_{U}(U;\theta)} \right\|
		\leq  \frac{ \left\| \nabla_{\theta} q_{U}(U;\theta) \right\|}{q_{U}(U;\theta)} 
		\leq \frac{\tilde{L}}{\tilde{q}_0}, \forall U\in\mbR^{n+1},
	\end{align}  
	and 
	\begin{align}
		\left\|  \mbE_{p_{U}}
		\left[ \frac{\nabla_{\theta} q_{U}(U;\theta)}{q_{U}(U;\theta)} \right] \right\| 
		& \leq  \mbE_{p_{U}}
		\left[ \left\|   \frac{\nabla_{\theta} q_{U}(U;\theta)}{q_{U}(U;\theta)} \right\|  \right] \label{Eq:JensenIneq3}\\
		& =  \mbE_{p_{U}}
		\left[  \frac{ \left\|  \nabla_{\theta} q_{U}(U;\theta) \right\|  }{q_{U}(U;\theta)} \right] \label{Eq:PositiveProb}\\
		& \leq  \mbE_{p_{U}}
		\left[ \frac{ \tilde{L} }{\tilde{q}_0}\right] \label{Eq:SumdUpbd}\\
		& \leq \frac{\tilde{L}}{\tilde{q}_0}, \label{Eq:ExpcttUpBd}
	\end{align}
	where the \eqref{Eq:JensenIneq3} is due to Jensen's inequality applied to convex norm function, the \eqref{Eq:PositiveProb} is due to the positivenss assumption on the probability, and the \eqref{Eq:SumdUpbd} is due to \eqref{Eq:SmdUpBd}.
	
	Combining the above with \eqref{Eq:BdnessTildeS} gives
	\begin{align}\label{Eq:BdnessTildeS2}
		\|\tilde{S}_i\|
		\leq \frac{1}{N} \left( \mbE_{p_{U}}
		\left[ \left\|  \frac{\nabla_{\theta} q_{U}(U;\theta)}{q_{U}(U;\theta)} \right\| \right] + \left\| \frac{\nabla_{\theta}q_{U}(U_i;\theta) }{q_{U}(U_i;\theta)} \right\| \right) 
		\leq \frac{2\tilde{L}}{N\tilde{q}_0}.
	\end{align}\label{Eq:Bd}
	
	Since 
	\begin{align}\label{Eq:TildeSTildeST}
		\tilde{S}_i \tilde{S}_i ^T
		& = \frac{1}{N^2} \left(\mbE_{p_{U}}
		\left[ \frac{\nabla_{\theta} q_{U}(U;\theta)}{q_{U}(U;\theta)} \right] 
		- \frac{\nabla_{\theta}q_{U}(U_i;\theta) }{q_{U}(U_i;\theta)} \right)  
		\left(\mbE_{p_{U}}
		\left[ \frac{\nabla_{\theta} q_{U}(U;\theta)}{q_{U}(U;\theta)} \right] 
		- \frac{\nabla_{\theta}q_{U}(U_i;\theta) }{q_{U}(U_i;\theta)} \right) ^T \\
		& = \frac{1}{N^2} \left( \tilde{A}_{i1} + \tilde{A}_{i2} + \tilde{A}_{i3} + \tilde{A}_{i4} \right), 
	\end{align}
	where 
	\begin{align}
		\tilde{A}_{i1}: = \mbE_{p_{U}}
		\left[ \frac{\nabla_{\theta} q_{U}(U;\theta)}{q_{U}(U;\theta)} \right] 
		\mbE_{p_{U}}
		\left[ \frac{\nabla_{\theta} q_{U}(U;\theta)}{q_{U}(U;\theta)} \right] ^T, \\
		\tilde{A}_{i2}:= - \mbE_{p_{U}}
		\left[ \frac{\nabla_{\theta} q_{U}(U;\theta)}{q_{U}(U;\theta)} \right] \frac{\nabla_{\theta}q_{U}(U_i;\theta) }{q_{U}(U_i;\theta)}^T, \\
		\tilde{A}_{i3}:= - \frac{\nabla_{\theta}q_{U}(U_i;\theta) }{q_{U}(U_i;\theta)} 
		\mbE_{p_{U}}
		\left[ \frac{\nabla_{\theta} q_{U}(U;\theta)}{q_{U}(U;\theta)} \right] ^T, \\
		\tilde{A}_{i4}:=  \frac{\nabla_{\theta}q_{U}(U_i;\theta) }{q_{U}(U_i;\theta)}\frac{\nabla_{\theta}q_{U}(U_i;\theta) }{q_{U}(U_i;\theta)}^T, 
	\end{align}
	then 
	\begin{align}\label{Eq:Variance}
		\left\| \mbE_{p_\mcS}\left[ \tilde{Z}\tilde{Z}^T\right]  \right\| 
		& = \left\| \mbE_{p_\mcS}\left[ \left(\sum_{i=1}^N \tilde{S}_i \right) \left( \sum_{i=1}^N \tilde{S}_i \right)^T\right]  \right\| \\ 
		& = \left\| \mbE_{p_\mcS}\left[ \sum_{i=1}^N \tilde{S}_i \tilde{S}_i ^T\right]  \right\| \label{Eq:IIDAssumption}\\ 
		& = \left\| \mbE_{p_\mcS}\left[ \sum_{i=1}^N \frac{1}{N^2} \left( \tilde{A}_{i1} + \tilde{A}_{i2} + \tilde{A}_{i3} + \tilde{A}_{i4} \right)\right]  \right\| \\ 
		& = \frac{1}{N^2} \left\| \sum_{i=1}^N  \mbE_{p_{U_i}}\left[ \tilde{A}_{i1} + \tilde{A}_{i2} + \tilde{A}_{i3} + \tilde{A}_{i4} \right]  \right\| \label{Eq:IIDAssumption2}\\ 
		& \leq \frac{1}{N^2} \sum_{i=1}^N \left( \left\|\mbE_{p_{U_i}}\left[ \tilde{A}_{i1}\right] \right\| 
		+ \left\|\mbE_{p_{U_i}}\left[ \tilde{A}_{i2}\right] \right\|
		+ \left\|\mbE_{p_{U_i}}\left[ \tilde{A}_{i3}\right] \right\|
		+ \left\|\mbE_{p_{U_i}}\left[ \tilde{A}_{i4}\right] \right\|\right),
	\end{align}
	where the norm for matrix is operator norm.
	
	Notice that
	\begin{align}
		\left\|\mbE_{p_{U_i}}\left[ \tilde{A}_{i1}\right] \right\| 
		& = \left\|\mbE_{p_{U_i}}\left[  \mbE_{p_{U}}
		\left[ \frac{\nabla_{\theta} q_{U}(U;\theta)}{q_{U}(U;\theta)} \right] 
		\mbE_{p_{U}}
		\left[ \frac{\nabla_{\theta} q_{U}(U;\theta)}{q_{U}(U;\theta)} \right] ^T \right] \right\| \label{Eq:DefinitionAi1}\\
		& =  \left\|  \mbE_{p_{U}}
		\left[ \frac{\nabla_{\theta} q_{U}(U;\theta)}{q_{U}(U;\theta)} \right] 
		\mbE_{p_{U}}
		\left[ \frac{\nabla_{\theta} q_{U}(U;\theta)}{q_{U}(U;\theta)} \right] ^T\right\| \\
		& =  \left\|  \mbE_{p_{U}}
		\left[ \frac{\nabla_{\theta} q_{U}(U;\theta)}{q_{U}(U;\theta)} \right] \right\| 
		\left\| \mbE_{p_{U}}
		\left[ \frac{\nabla_{\theta} q_{U}(U;\theta)}{q_{U}(U;\theta)} \right]\right\| \label{Eq:VectorOuterProdNorm} \\
		& \leq \frac{\tilde{L}^2}{\tilde{q}_0^2} \label{Eq:SummandUpperBound}
	\end{align}
	where the \eqref{Eq:SummandUpperBound} is due to \eqref{Eq:ExpcttUpBd}. 
	
	Similarly, we have: 
	\begin{align}
		\left\|\mbE_{p_{U_i}}\left[ \tilde{A}_{i2}\right] \right\| 
		& = \left\|\mbE_{p_{U_i}}\left[  - \mbE_{p_{U}}
		\left[ \frac{\nabla_{\theta} q_{U}(U;\theta)}{q_{U}(U;\theta)} \right] \frac{\nabla_{\theta}q_{U}(U_i;\theta) }{q_{U}(U_i;\theta)}^T \right] \right\| \\
		& \leq \mbE_{p_{U_i}}\left[  \left\|  \mbE_{p_{U}}
		\left[ \frac{\nabla_{\theta} q_{U}(U;\theta)}{q_{U}(U;\theta)} \right] \frac{\nabla_{\theta}q_{U}(U_i;\theta) }{q_{U}(U_i;\theta)}^T  \right\| \right] \label{Eq:JensenIneq2}\\
		& = \mbE_{p_{U_i}}\left[  \left\|  \mbE_{p_{U}}
		\left[ \frac{\nabla_{\theta} q_{U}(U;\theta)}{q_{U}(U;\theta)} \right] \right\| 
		\left\| \frac{\nabla_{\theta}q_{U}(U_i;\theta) }{q_{U}(U_i;\theta)} \right\| \right]\\
		& \leq \mbE_{p_{U_i}}\left[  \frac{ \tilde{L} }{\tilde{q}_0}  \frac{ \tilde{L} }{\tilde{q}_0}  \right] 
		\label{Eq:NormUpBd}
		\leq \frac{\tilde{L}^2}{\tilde{q}_0^2},
	\end{align}
	where the \eqref{Eq:NormUpBd} is due to \eqref{Eq:SmdUpBd} and \eqref{Eq:ExpcttUpBd},
	\begin{align}
		\left\|\mbE_{p_{U_i}}\left[ \tilde{A}_{i3}\right] \right\| 
		& = \left\| \mbE_{p_{U_i}} \left[  - \frac{\nabla_{\theta}q_{U}(U_i;\theta) }{q_{U}(U_i;\theta)} 
		\mbE_{p_{U}}
		\left[ \frac{\nabla_{\theta} q_{U}(U;\theta)}{q_{U}(U;\theta)} \right] ^T\right] \right\| \\
		& \leq  \mbE_{p_{U_i}} \left[  \left\|  \frac{\nabla_{\theta}q_{U}(U_i;\theta) }{q_{U}(U_i;\theta)} 
		\mbE_{p_{U}}
		\left[ \frac{\nabla_{\theta} q_{U}(U;\theta)}{q_{U}(U;\theta)} \right]\right\|\right] \\
		& =  \mbE_{p_{U_i}} \left[  \left\|  \frac{\nabla_{\theta}q_{U}(U_i;\theta) }{q_{U}(U_i;\theta)} \right\|
		\left\| \mbE_{p_{U}}
		\left[ \frac{\nabla_{\theta} q_{U}(U;\theta)}{q_{U}(U;\theta)}\right] \right\|\right] 
		\leq \frac{\tilde{L}^2}{\tilde{q}_0^2} \label{Eq:normUpbd},
	\end{align}
	where the \eqref{Eq:normUpbd} is due to \eqref{Eq:SmdUpBd} and \eqref{Eq:ExpcttUpBd}, and
	\begin{align}
		\left\|\mbE_{p_{U_i}}\left[ \tilde{A}_{i4}\right] \right\| 
		& = \left\| \mbE_{p_{U_i}}\left[ \frac{\nabla_{\theta}q_{U}(U_i;\theta) }{q_{U}(U_i;\theta)}\frac{\nabla_{\theta}q_{U}(U_i;\theta) }{q_{U}(U_i;\theta)}^T\right] \right\| \\ 
		& \leq \mbE_{p_{U_i}} \left[ \left\| \frac{\nabla_{\theta}q_{U}(U_i;\theta) }{q_{U}(U_i;\theta)}\frac{\nabla_{\theta}q_{U}(U_i;\theta) }{q_{U}(U_i;\theta)}^T \right\| \right] \\ 
		& = \mbE_{p_{U_i}} \left[ \left\| \frac{\nabla_{\theta}q_{U}(U_i;\theta) }{q_{U}(U_i;\theta)}\right\| \left\|\frac{\nabla_{\theta}q_{U}(U_i;\theta) }{q_{U}(U_i;\theta)} \right\| \right] \\ 
		& \leq  \mbE_{p_{U_i}} \left[  \frac{ \tilde{L} }{\tilde{q}_0}  \frac{ \tilde{L} }{\tilde{q}_0}\right] \label{Eq:smdUpBd}
		\leq  \frac{\tilde{L}^2}{\tilde{q}_0^2}.
	\end{align}
	where the \eqref{Eq:smdUpBd} is due to \eqref{Eq:SmdUpBd}. Thus, 
	\begin{align}\label{Eq:Variance_2}
		\left\| \mbE_{p_\mcS}\left[ \tilde{Z}\tilde{Z}^T\right]  \right\| 
		\leq \frac{1}{N^2} \sum_{i=1}^N \left(\frac{\tilde{L}^2}{\tilde{q}_0^2} + \frac{\tilde{L}^2}{\tilde{q}_0^2} + \frac{\tilde{L}^2}{\tilde{q}_0^2} + \frac{\tilde{L}^2}{\tilde{q}_0^2} \right)
		\leq \frac{4}{N}\frac{\tilde{L}^2}{\tilde{q}_0^2}.
	\end{align}
	
	Similarly, we can derive the upper bound for $\left\| \mbE_{p_\mcS} \left[\tilde{Z}^T \tilde{Z}\right] \right\|$. Since
	\begin{align}\label{Eq:TildeSTildeS}
		\tilde{S}_i^T \tilde{S}_i 
		& = \frac{1}{N^2} \left(\mbE_{p_{U}}
		\left[ \frac{\nabla_{\theta} q_{U}(U;\theta)}{q_{U}(U;\theta)} \right] 
		- \frac{\nabla_{\theta}q_{U}(U_i;\theta) }{q_{U}(U_i;\theta)} \right)^T 
		\left(\mbE_{p_{U}}
		\left[ \frac{\nabla_{\theta} q_{U}(U;\theta)}{q_{U}(U;\theta)} \right] 
		- \frac{\nabla_{\theta}q_{U}(U_i;\theta) }{q_{U}(U_i;\theta)} \right)  \\
		& = \frac{1}{N^2} \left( \tilde{A}_{i5} + \tilde{A}_{i6} + \tilde{A}_{i7}\right), 
	\end{align}
	where we define 
	\begin{align}
		\tilde{A}_{i5}:= \left\| \mbE_{p_{U}}
		\left[ \frac{\nabla_{\theta} q_{U}(U;\theta)}{q_{U}(U;\theta)} \right]  \right\|^2, \\
		\tilde{A}_{i6}: = - 2 \mbE_{p_{U}}
		\left[ \frac{\nabla_{\theta} q_{U}(U;\theta)}{q_{U}(U;\theta)} \right]^T  \frac{\nabla_{\theta}q_{U}(U_i;\theta) }{q_{U}(U_i;\theta)}, \\
		\tilde{A}_{i7}:= \left\| \frac{\nabla_{\theta}q_{U}(U_i;\theta) }{q_{U}(U_i;\theta)} \right\|^2,
	\end{align}
	then 
	\begin{align}\label{Eq:tildeZTtildeZ}
		\left\| \mbE_{p_\mcS} \left[\tilde{Z}^T \tilde{Z}\right] \right\|
		& = \left\| \mbE_{p_\mcS} \left[ \left(\sum_{i=1}^N \tilde{S}_i \right)^T \left(\sum_{i=1}^N \tilde{S}_i \right) \right] \right\| \\
		& = \left\| \mbE_{p_\mcS} \left[ \sum_{i=1}^N \tilde{S}_i^T \tilde{S}_i \right] \right\| \\
		& = \left\| \mbE_{p_\mcS} \left[ \sum_{i=1}^N \frac{1}{N^2} \left( \tilde{A}_{i5} + \tilde{A}_{i6} + \tilde{A}_{i7}\right) \right] \right\| \\
		& = \frac{1}{N^2}  \left\| \mbE_{p_\mcS} \left[ \sum_{i=1}^N \left( \tilde{A}_{i5} + \tilde{A}_{i6} + \tilde{A}_{i7}\right) \right] \right\| \\
		& = \frac{1}{N^2}  \left\|  \sum_{i=1}^N  \mbE_{p_{U_i}} \left[ \left( \tilde{A}_{i5} + \tilde{A}_{i6} + \tilde{A}_{i7}\right) \right] \right\| \\
		& \leq \frac{1}{N^2}  \sum_{i=1}^N   \left\| \mbE_{p_{U_i}} \left[ \left( \tilde{A}_{i5} + \tilde{A}_{i6} + \tilde{A}_{i7}\right) \right] \right\| \\
		& \leq \frac{1}{N^2}  \sum_{i=1}^N   \left( \left\| \mbE_{p_{U_i}} \left[\tilde{A}_{i5}\right]\right\| 
		+  \left\| \mbE_{p_{U_i}} \left[\tilde{A}_{i6}\right]\right\| 
		+  \left\| \mbE_{p_{U_i}} \left[\tilde{A}_{i7} \right] \right\| \right).
	\end{align}
	
	Since 
	\begin{align}
		\left\| \mbE_{p_{U_i}} \left[\tilde{A}_{i5}\right]\right\| 
		& = \left\| \mbE_{p_{U_i}} \left[ \left\| \mbE_{p_{U}}
		\left[ \frac{\nabla_{\theta} q_{U}(U;\theta)}{q_{U}(U;\theta)} \right]  \right\|^2 \right]\right\| \\
		& = \left\| \mbE_{p_{U}}
		\left[ \frac{\nabla_{\theta} q_{U}(U;\theta)}{q_{U}(U;\theta)} \right]  \right\|^2 
		\leq \frac{ \tilde{L}^2 }{\tilde{q}_0^2} \label{Eq:expctsmdUpbd}, 
	\end{align}
	where the \eqref{Eq:expctsmdUpbd} is due to \eqref{Eq:ExpcttUpBd},
	\begin{align}
		\left\| \mbE_{p_{U_i}} \left[\tilde{A}_{i6}\right]\right\| 
		& = \left\| \mbE_{p_{U_i}} \left[ - 2 \mbE_{p_{U}}
		\left[ \frac{\nabla_{\theta} q_{U}(U;\theta)}{q_{U}(U;\theta)} \right]^T  \frac{\nabla_{\theta}q_{U}(U_i;\theta) }{q_{U}(U_i;\theta)} \right]\right\| \\
		& \leq \mbE_{p_{U_i}} \left[  \left\| 2 \mbE_{p_{U}}
		\left[ \frac{\nabla_{\theta} q_{U}(U;\theta)}{q_{U}(U;\theta)} \right]^T  \frac{\nabla_{\theta}q_{U}(U_i;\theta) }{q_{U}(U_i;\theta)} \right\| \right]\\
		& \leq 2 \mbE_{p_{U_i}} \left[  \left\| \mbE_{p_{U}}
		\left[ \frac{\nabla_{\theta} q_{U}(U;\theta)}{q_{U}(U;\theta)} \right] \right\| \left\|  \frac{\nabla_{\theta}q_{U}(U_i;\theta) }{q_{U}(U_i;\theta)} \right\| \right] \label{Eq:CSIneq} \\
		& \leq 2 \mbE_{p_{U_i}} \left[  \frac{ \tilde{L} }{\tilde{q}_0}  \frac{ \tilde{L} }{\tilde{q}_0}  \right] \label{Eq:CSInequ} \\
		& = 2\frac{ \tilde{L}^2 }{\tilde{q}_0^2}\label{eq:normUpBd}, 
	\end{align}
	where the \eqref{Eq:CSInequ} is due to Cauchy Schwartz inequality and the \eqref{eq:normUpBd} is due to \eqref{Eq:SmdUpBd} and \eqref{Eq:ExpcttUpBd},
	and 
	\begin{align}
		\left\| \mbE_{p_{U_i}} \left[\tilde{A}_{i7}\right]\right\| 
		& = \left\| \mbE_{p_{U_i}} \left[ \left\| \frac{\nabla_{\theta}q_{U}(U_i;\theta) }{q_{U}(U_i;\theta)} \right\|^2 \right]\right\| \\
		& = \mbE_{p_{U_i}} \left[ \left\| \frac{\nabla_{\theta}q_{U}(U_i;\theta) }{q_{U}(Y_i|X_i;\theta)} \right\|^2 \right] \\
		& \leq \mbE_{p_{U_i}} \left[ \frac{ \tilde{L}^2 }{\tilde{q}_0^2} \right] \\
		& = \frac{ \tilde{L}^2 }{\tilde{q}_0^2}, \label{eq:smdUpbd2}
	\end{align}
	where \eqref{eq:smdUpbd2} is due to \eqref{Eq:SmdUpBd}, we have from \eqref{Eq:tildeZTtildeZ}
	\begin{align}\label{Eq:Variance_3}
		\left\| \mbE_{p_\mcS} \left[\tilde{Z}^T \tilde{Z}\right] \right\|
		& \leq \frac{1}{N^2}  \sum_{i=1}^N   \left( \left\| \mbE_{p_{U_i}} \left[\tilde{A}_{i5}\right]\right\| 
		+  \left\| \mbE_{p_{U_i}} \left[\tilde{A}_{i6}\right]\right\| 
		+  \left\| \mbE_{p_{X_i,Y_i}} \left[\tilde{A}_{i7} \right] \right\| \right) \\
		& \leq 	\frac{1}{N^2}  \sum_{i=1}^N   \left( \frac{ \tilde{L}^2 }{\tilde{q}_0^2}
		+ 2\frac{ \tilde{L}^2 }{\tilde{q}_0^2}
		+  \frac{ \tilde{L}^2 }{\tilde{q}_0^2}\right) \\
		& = \frac{4}{N} \frac{ \tilde{L}^2 }{\tilde{q}_0^2}.
	\end{align}
	
	Combining \eqref{Eq:Variance_2} and \eqref{Eq:Variance_3} gives
	\begin{align}
		v(\tilde{Z}) \leq \frac{4}{N} \frac{ \tilde{L}^2 }{\tilde{q}_0^2},
	\end{align}
	where $v(\tilde{Z})$ is defined as
	\begin{align}
		v(\tilde{Z}): = \max(\left\| \mbE_{p_\mcS} \left[\tilde{Z}\tilde{Z}^T\right] \right\|, \left\| \mbE_{p_\mcS} \left[\tilde{Z}^T \tilde{Z}\right] \right\|).
	\end{align}
	
	From Lemma \ref{Lem:MtrxBersteinConcIneq}, we have 
	\begin{align}
		\mbP(\|\tilde{Z}\| \geq t) 
		\leq (m+1) \exp\left( - \frac{t^2/2}{ v(\tilde{Z}) +  \frac{2\tilde{L}}{N\tilde{q}_0} t/3}   \right) 
		\leq (m+1) \exp\left( - \frac{t^2/2}{ \frac{4}{N} \frac{ \tilde{L}^2 }{\tilde{q}_0^2} +  \frac{2\tilde{L}}{N\tilde{q}_0} t/3}   \right),
	\end{align}
	and 
	\begin{align}
		\mbE_{p_\mcS}[\|\tilde{Z}\|] 
		& \leq \sqrt{2v(\tilde{Z}) \log(m+1)} + \frac{1}{3} \frac{2\tilde{L}}{N\tilde{q}_0} \log(dm+1) \\
		& \leq \frac{2\tilde{L}}{\tilde{q}_0}\left( \sqrt{\frac{2}{N}\log(m+1)} + \frac{\log(m+1)}{3N} \right).
	\end{align}
	Thus, 
	\begin{align}
		\mbP(\|\tilde{Z}\|^2 \geq t) 
		\leq (m+1) \exp\left( - \frac{t/2}{ \frac{4}{N} \frac{ \tilde{L}^2 }{\tilde{q}_0^2} +  \frac{2\tilde{L}}{N\tilde{q}_0} \sqrt{t} /3}   \right).
	\end{align}
	When $N\geq \left( \frac{2\tilde{L}^2}{\tilde{q}_0^2 t} + \frac{4\tilde{L}}{3 \tilde{q}_0 \sqrt{t}} \right)\log\frac{m+1}{\delta}$, we have
	\begin{align}
		\mbP(\|\tilde{Z}\|^2 \leq t) \geq 1-\delta,
	\end{align}
	where $t, \delta>0$ are constants.
	
\end{proof}

We now present the detailed proof of Lemma 4.3 which is restated as in Lemma \ref{Thm:GradConcentrationIneq}.

\begin{thm}\label{Thm:GradConcentrationIneq}
	(Concentration Inequality for Mutual Information Loss Function Gradient) We consider the $\|\nabla\mcL^{(N)}(\theta_t) - \nabla \mcL(\theta_t) \|^2$ in each iteration of the Algorithm \ref{Alg:SGD} in Theorem \ref{Thm:cnvgcRt} where $\theta_t$ contains $\theta_{Y|X}\in\mbR^m$ and $\theta_Y\in\mbR^{m'}$. We assume that $q_{Y|X}(Y|X;\theta_{Y|X})$ is $\tilde{L}$-Lipschitz continuous with respect to $(X,Y)$, i.e., 
	\begin{align} 
		\|q_{Y|X}(Y'|X'; \theta_{Y|X}) - q_{Y|X}(Y|X; \theta_{Y|X})\| 
		\leq \tilde{L} \left\| 
		\left[\begin{matrix}
			X' \\ Y'
		\end{matrix}\right]
		- 
		\left[\begin{matrix}
			X \\ Y
		\end{matrix}\right]
		\right\|, \forall (X,Y), (X',Y')\in\mbR^{m}\times\mbR, 
	\end{align}
	and that $q_{Y|X}(Y|X;\theta_{Y|X})$ does not vanish, i.e., 
	\begin{align} 
		q_{Y|X}(Y|X;\theta_{Y|X}) \geq \tilde{q}_0, \forall  (X,Y)\in\mbR^{m}\times\mbR,
	\end{align}  
	where $\tilde{L}>0$ and $\tilde{q}_0>0$ are constants. We also assume that $q_{Y}(Y;\theta_{Y})$ is $\bar{L}$-Lipschitz continuous with respect to $Y$, i.e., 
	\begin{align} 
		\|q_{Y}(Y'; \theta_Y) - q_{Y}(Y; \theta_Y)\| 
		\leq \bar{L} \left\| 
		Y' - Y
		\right\|, \forall Y,Y'\in\mbR, 
	\end{align}
	and that $q_{Y}(Y;\theta_{Y})$ does not vanish, i.e., 
	\begin{align} 
		q_{Y}(Y;\theta_Y) \geq \bar{q}_0, \forall  Y\in\mbR,
	\end{align}  
	where $\bar{L}>0$ and $\bar{q}_0>0$ are constants. For any $\epsilon>0, \delta>0$, if 
	\begin{align} 
		N\geq \max\left(\left(\frac{4\tilde{L}^2}{\tilde{q}_0^2\epsilon} + \frac{4\sqrt{2}\tilde{L}}{3\tilde{q}_0\sqrt{\epsilon}} \right)\log\frac{2(m+1)}{\delta} ,
		\left( \frac{4\bar{L}^2\lambda_{ent}^2}{\bar{q}_0^2\epsilon} + \frac{4\sqrt{2}\bar{L}\lambda_{ent}}{3\bar{q}_0 \sqrt{\epsilon}} \right) \log\frac{2(m'+1)}{\delta}\right), 
	\end{align}
	we have
	\begin{align}
		\mbP_{\mcS}\left( \|\nabla\mcL^{(N)}(\theta_t) - \nabla \mcL(\theta_t) \|^2 \leq \epsilon \right) \geq 1-\delta.
	\end{align}
\end{thm}

\begin{proof}\label{Prf:GradConcentrationIneq}
	(of Theorem \ref{Thm:GradConcentrationIneq}) From the definitions of $\mcL(\theta)$ in \eqref{Defn:DistributionalLoss} and $\mcL^{(N)}(\theta)$ in Algorithm \ref{Alg:SGD}, we have 
	\begin{align}
		\| \nabla \mcL^{(N)}(\theta_t) - \nabla \mcL(\theta_t)\|^2 
		& = \left\| 
		\left[ \begin{matrix}
			\frac{1}{N} \sum_{i=1}^N - \frac{\nabla_{\theta_{Y|X}}q_{Y|X}(Y_i|X_i;\theta_{Y|X}) }{q_{Y|X}(Y_i|X_i;\theta_{Y|X})} \\
			\frac{1}{N} \sum_{i=1}^N - \lambda_{ent} \frac{\nabla_{\theta_Y} q_Y(Y_i; \theta_Y)}{ q_Y(Y_i; \theta_Y)}
		\end{matrix} \right]
		- 
		\left[\begin{matrix}
			\mbE_{p_{X,Y}}\left[ - \frac{\nabla_{\theta_{Y|X}} q_{Y|X}(Y|X;\theta_{Y|X})}{q_{Y|X}(Y|X;\theta_{Y|X})} \right] \\ 
			\lambda_{ent} \mbE_{p_Y} \left[ - \frac{\nabla_{\theta_Y} q_Y(Y;\theta_Y)}{q_Y(Y;\theta_Y)} \right]
		\end{matrix}\right]
		\right\|^2  \\
		& = \left\| 
		\left[ \begin{matrix}
			\frac{1}{N} \sum_{i=1}^N \left(\mbE_{p_{X,Y}}
			\left[ \frac{\nabla_{\theta_{Y|X}} q_{Y|X}(Y|X;\theta_{Y|X})}{q_{Y|X}(Y|X;\theta_{Y|X})} \right] - \frac{\nabla_{\theta_{Y|X}}q_{Y|X}(Y_i|X_i;\theta_{Y|X}) }{q_{Y|X}(Y_i|X_i;\theta_{Y|X})} \right) \\
			\frac{1}{N} \sum_{i=1}^N \left( \lambda_{ent} \mbE_{p_Y} \left[ \frac{\nabla_{\theta_Y} q_Y(Y;\theta_Y)}{q_Y(Y;\theta_Y)} \right] - \lambda_{ent} \frac{\nabla_{\theta_Y} q_Y(Y_i; \theta_Y)}{ q_Y(Y_i; \theta_Y)}\right)
		\end{matrix}\right] \right\|^2 \\
		& =  \tilde{I} +  \lambda_{ent}^2 \bar{I},
	\end{align}
	where the expectation is with respect to $\mcS:=\{(X_i,Y_i)\}_{i=1}^N$, and we define
	\begin{align}
		\tilde{I}:= \left\| \frac{1}{N} \sum_{i=1}^N \left(\mbE_{p_{X,Y}}
		\left[ \frac{\nabla_{\theta_{Y|X}} q_{Y|X}(Y|X;\theta_{Y|X})}{q_{Y|X}(Y|X;\theta_{Y|X})} \right] - \frac{\nabla_{\theta_{Y|X}}q_{Y|X}(Y_i|X_i;\theta_{Y|X}) }{q_{Y|X}(Y_i|X_i;\theta_{Y|X})} \right) \right\|^2 ,
	\end{align}
	and 
	\begin{align}
		\bar{I} := \left\| \frac{1}{N} \sum_{i=1}^N \left( \mbE_{p_Y} \left[ \frac{\nabla_{\theta_Y} q_Y(Y;\theta_Y)}{q_Y(Y;\theta_Y)} \right] - \frac{\nabla_{\theta_Y} q_Y(Y_i; \theta_Y)}{ q_Y(Y_i; \theta_Y)}\right) \right\|^2.
	\end{align}
	In the following, we will establish concentration inequalities for both $\tilde{I}$ and $\bar{I}$ via Lemma \ref{Lem:Conineq_Log-Prob-Grad}. 
	
	For $\tilde{I}$, we can treat $\left[ \begin{matrix}
		X_i \\ Y_i
	\end{matrix} \right]\in\mbR^{n+1}$ as $U_i$ for $i=1,\cdots,N$. Then, from Lemma \ref{Lem:Conineq_Log-Prob-Grad}, we have for any $\theta_{Y|X}\in\mbR^m$, if $N\geq \left( \frac{2\tilde{L}^2}{\tilde{q}_0^2 t} + \frac{4\tilde{L}}{3 \tilde{q}_0 \sqrt{t}} \right)\log\frac{m+1}{\delta}$, we have $\mbP(\tilde{I} \leq t) \geq 1-\delta$, where $t, \delta>0$ are arbitrary constants. Taking $t=\frac{\epsilon}{2}$, we get that for any $\theta_{Y|X}$, if $N\geq \left(\frac{4\tilde{L}^2}{\tilde{q}_0^2\epsilon} + \frac{4\sqrt{2}\tilde{L}}{3\tilde{q}_0\sqrt{\epsilon}} \right)\log\frac{2(m+1)}{\delta}$, then
	\begin{align}\label{Eq:ConcentrationTildeI}
		\mbP_{\mcS}\left(\tilde{I} \leq \frac{\epsilon}{2}\right) \geq 1-\frac{\delta}{2},
	\end{align}
	or 
	\begin{align}\label{Eq:ConcentrationTildeI_quiv}
		\mbP_{\mcS}\left(\tilde{I} \geq \frac{\epsilon}{2}\right) \leq \frac{\delta}{2}.
	\end{align}
	
	Similarly, for $\bar{I}$, we can treat $Y_i$ as $U_i$ for $i=1,\cdots,N$. Then, from Lemma \ref{Lem:Conineq_Log-Prob-Grad}, we have for any $\theta_Y\in\mbR^{m'}$, if $N\geq \left( \frac{2\bar{L}^2}{\bar{q}_0^2 t} + \frac{4\bar{L}}{3 \bar{q}_0 \sqrt{t}} \right)\log\frac{m'+1}{\delta}$, we have	$\mbP(\bar{I} \leq t) \geq 1-\delta$, where $t, \delta>0$ are arbitrary constants. Taking $t=\frac{\epsilon}{2\lambda_{ent}^2}$, we have that for any $\theta_Y$, if $N\geq \left( \frac{4\bar{L}^2\lambda_{ent}^2}{\bar{q}_0^2\epsilon} + \frac{4\sqrt{2}\bar{L}\lambda_{ent}}{3\bar{q}_0 \sqrt{\epsilon}} \right) \log\frac{2(m'+1)}{\delta}$, then 
	\begin{align}\label{Eq:ConcentrationBarI}
		\mbP_{\mcS}\left(\bar{I} \leq \frac{\epsilon}{2}\right) \geq 1-\frac{\delta}{2},
	\end{align}
	or 
	\begin{align}\label{Eq:ConcentrationBarI_equiv}
		\mbP_{\mcS}\left(\bar{I} \geq \frac{\epsilon}{2}\right) \leq \frac{\delta}{2}.
	\end{align}
	
	Thus, we have for any $\theta_{Y|X}, \theta_Y$, if 
	\begin{align} 
		N\geq \max\left(\left(\frac{4\tilde{L}^2}{\tilde{q}_0^2\epsilon} + \frac{4\sqrt{2}\tilde{L}}{3\tilde{q}_0\sqrt{\epsilon}} \right)\log\frac{2(m+1)}{\delta} ,
		\left( \frac{4\bar{L}^2\lambda_{ent}^2}{\bar{q}_0^2\epsilon} + \frac{4\sqrt{2}\bar{L}\lambda_{ent}}{3\bar{q}_0 \sqrt{\epsilon}} \right) \log\frac{2(m'+1)}{\delta}\right), 
	\end{align}
	then
	\begin{align}
		\mbP_{\mcS}\left( \|\nabla\mcL^{(N)}(\theta_t) - \nabla \mcL(\theta_t) \|^2 \leq \epsilon \right)
		&  = 	\mbP_{\mcS}\left( \tilde{I} + \bar{I} \leq \epsilon \right) \\
		& \geq \mbP_{\mcS}\left( \tilde{I} \leq \frac{\epsilon}{2}, \bar{I} \leq \frac{\epsilon}{2} \right) \\ 
		& = 1-  \mbP_{\mcS}\left( \tilde{I} \geq \frac{\epsilon}{2}, \text{or}\ \bar{I} \geq \frac{\epsilon}{2} \right) \\ 
		& \geq 1-  \mbP_{\mcS}\left( \tilde{I} \geq \frac{\epsilon}{2}\right)- \mbP_{\mcS}\left( \bar{I} \geq \frac{\epsilon}{2} \right) \label{Eq:UnionBound} \\ 
		& \geq 1-\frac{\delta}{2}-\frac{\delta}{2} \label{Eq:DeviationBound}\\
		& = 1-\delta,
	\end{align}
	where \eqref{Eq:UnionBound} is due to union bound, and \eqref{Eq:DeviationBound} is due to \eqref{Eq:ConcentrationTildeI_quiv} and \eqref{Eq:ConcentrationBarI_equiv}.
	
\end{proof}

We now consider a data model over correlated joint Gaussian distribution $p_{X,Y}$, i.e., 
\begin{align}\label{Defn:RepresentationLearningDataModel}
	X=\rho Y + \sqrt{1-\rho^2} Z, \rho\in(0,1)
\end{align} 
where the elements of $Y\in\mbR^n$ follow I.I.D. standard Gaussian distribution $\mcN(0,I_n)$, and the elements of $Z\in\mbR^n$ follow I.I.D. standard Gaussian distribution $\mcN(0,I_n)$, and the $Y$ and $Z$ are independent. We present the detailed proof of Theorem 5.1 which is restated as in Lemma \ref{Thm:MIGTinGaussian} and Corollary 1 which is restated as in Corollary \ref{Cor:GeneralizationLossViaMI}.

\begin{thm}\label{Thm:MIGTinGaussian}
	(Mutual Information of Multi-output Regression Data Model) We consider a multi-output regression task where the input $Y\in\mbR^n$ of a machine learning systems has all its elements folllowing I.I.D. standard Gaussian distribution, and the output $X\in\mbR^n$ is generated according to \eqref{Defn:RepresentationLearningDataModel}. Then we have 
	\begin{align}\label{Eq:MIGTinJointGuassian}
		I(X;Y) = \frac{n}{2} \log\frac{1}{1-\rho^2}.
	\end{align}
\end{thm}

\begin{proof}
	(of Theorem \ref{Thm:MIGTinGaussian}) Since $X=\rho Y + \sqrt{1-\rho^2} W$ where the elements of $Y\in\mbR^n$ and $W\in\mbR^n$ follow IID standard Gaussian distribution, then 
	\begin{align}
		\mbE[X] = \mbE[\rho Y +  \sqrt{1-\rho^2} W] = \bm{0}\in\mbR^{n},
	\end{align}
	and 
	\begin{align}
		\cov(X,X) 
		& = \mbE[(X - \mbE[X])(X - \mbE[X])^T] \\
		& = \mbE[XX^T]\\
		& = \mbE[(\rho Y +  \sqrt{1-\rho^2} W)(\rho Y +  \sqrt{1-\rho^2} W)^T]\\
		& = I_{n},
	\end{align}
	where $I_n\in\mbR^{n\times n}$ is an identity matrix. Define $Z:=\left[\begin{matrix}
		Y \\ X
	\end{matrix}\right] \in\mbR^{2n}$, and from 
	\begin{align}
		\cov(Y,X) = \mbE[(Y-\mbE[Y])(X - \mbE[X])^T] = \rho I_n,
	\end{align}
	we have the covariance matrix of the joint distribution $p_{UX}$ as 
	\begin{align}\label{Eq:JointGaussianStats}
		\Sigma_Z 
		:= \cov(Z,Z)
		=\left[\begin{matrix}
			I_n & \rho I_n \\
			\rho I_n & I_n
		\end{matrix}\right] \in\mbR^{2n\times 2n}, 
		\mu_Z := \mbE[Z] = \bm{0} \in\mbR^{2n}.
	\end{align}
	
	We now derive the mutual information (MI) $I_{p_{X,Y}}(X;Y)$ for $p_{X,Y}$ with covariance matrix and mean specified in \eqref{Eq:JointGaussianStats}. From the definition of MI, we have
	\begin{align}\label{Eq:MIGTinGaussian1}
		I(X;Y)
		& = D_{KL}(p_{X,Y}||p_Y p_X) \nonumber \\
		& = \int_z p_{X,Y}(z) \log\left(\frac{p_{X,Y}(z)}{p_Y(y) p_X(x)}\right) \nonumber  \\
		& = \int_z p_{X,Y}(z) \log\left( \frac{ \frac{1}{\sqrt{|2\pi\Sigma_Z|}} \exp\left(-\frac{1}{2} (z-\mu_Z)^T\Sigma_Z^{-1}(z-\mu_Z) \right) }
		{ \frac{1}{\sqrt{|2\pi\Sigma_Y|}} \exp\left(-\frac{1}{2} (u-\mu_Y)^T\Sigma_Y^{-1}(u-\mu_Y) \right) \frac{1}{\sqrt{|2\pi\Sigma_X|}} \exp\left(-\frac{1}{2} (x-\mu_X)^T\Sigma_X^{-1}(x-\mu_X) \right)}
		\right)  \nonumber  \\
		& = \int_z p_{X,Y}(z) \left( 
		\log\left( \frac{\sqrt{|2\pi\Sigma_X|} \sqrt{|2\pi\Sigma_Y|} }{\sqrt{|2\pi\Sigma_Z|}} \right)
		+
		\log\left( \frac{ \exp\left(-\frac{1}{2} (z-\mu_Z)^T\Sigma_Z^{-1}(z-\mu_Z) \right) }
		{ \exp\left(-\frac{1}{2} (z-\mu_Z)^T\Sigma_{Y,X}^{-1}(z-\mu_Z) \right)}
		\right)
		\right),
	\end{align}
	where we define
	\begin{align*}
		\Sigma_Y := \cov(Y,Y) = I_n, 
		\Sigma_X := \cov(X,X) = I_n,
		\Sigma_{Y,X} := \left[\begin{matrix}
			\Sigma_Y & \bm{0}\\
			\bm{0} & \Sigma_X
		\end{matrix}\right],
		\mu_Y := \mbE[Y]=\bm{0}, \mu_X := \mbE[X]=\bm{0},
	\end{align*}
	where $\bm{0}$ is a zero vector or matrix whose dimensionality can be determined according to the context.
	
	Notice that 
	\begin{align}\label{Eq:MIGTinGaussian2}
		\int_z p_{X,Y}(z)  
		\log\left( \frac{\sqrt{|2\pi\Sigma_X|} \sqrt{|2\pi\Sigma_Y|} }{\sqrt{|2\pi\Sigma_Z|}} \right)
		& = \log\left(\sqrt{\frac{|\Sigma_Y| |\Sigma_X|}{|\Sigma_Z|}} \right)  \nonumber \\
		& =  \log\left(\sqrt{\frac{ 1 }{|\Sigma_Z|}} \right)  \nonumber \\
		& = \log\left(\sqrt{\frac{ 1 }{|I_n - \rho I_n I_n^{-1} \rho I_n ||I|}} \right)  \nonumber \\
		& = -\frac{n}{2}\log(1-\rho^2),
	\end{align}
	where we used the following formula for computing the determinant of  a block matrix, i.e., for an arbitrary matrix \begin{align*}
		M=\left[\begin{matrix}
			A & B\\
			C & D
		\end{matrix}\right],
		A\in\mbR^{n\times n}, B\in\mbR^{n\times m}, C\in\mbR^{m\times n}, D\in\mbR^{m\times m},
	\end{align*}
	if the $D$ is invertible, then the determinant $|M|$ of $M$ is
	\begin{align}\label{Eq:BlockMatrixDeterminant}
		|M| = |A-BD^{-1}C||D|.
	\end{align}
	
	We now derive the second term in \eqref{Eq:MIGTinGaussian1} which can be simplified as follows
	\begin{align}\label{Eq:MIGTinGaussian3}
		\int_z p_{X,Y}(z) \log\left( \frac{ \exp\left(-\frac{1}{2} (z-\mu_Z)^T\Sigma_Z^{-1}(z-\mu_Z) \right) }
		{ \exp\left(-\frac{1}{2} (z-\mu_Z)^T\Sigma_{Y,X}^{-1}(z-\mu_Z) \right)}
		\right)
		& = - \frac{1}{2} \int_z p_{X,Y}(z) \left( (z-\mu_Z)^T\Sigma_Z^{-1}(z-\mu_Z) \right)  \nonumber \\
		& \quad  + \frac{1}{2} \int_z p_{X,Y}(z) \left( (z-\mu_Z)^T\Sigma_{Y,X}^{-1}(z-\mu_Z) \right)  \nonumber \\
		& = - \frac{1}{2} \int_z p_{X,Y}(z)  \tr\left((z-\mu_Z)^T\Sigma_Z^{-1}(z-\mu_Z)\right)  \nonumber \\
		& \quad  + \frac{1}{2} \int_z p_{X,Y}(z) \tr\left((z-\mu_Z)^T\Sigma_{Y,X}^{-1} (z-\mu_Z) \right)  \nonumber \\
		& = - \frac{1}{2} \int_z p_{X,Y}(z)  \tr\left((z-\mu_Z) (z-\mu_Z)^T\Sigma_Z^{-1}\right)  \nonumber \\
		& \quad  + \frac{1}{2} \int_z p_{X,Y}(z) \tr\left( (z-\mu_Z)(z-\mu_Z)^T\Sigma_{Y,X}^{-1} \right)  \nonumber  \\
		& = - \frac{1}{2}  \tr\left( \int_z p_{X,Y}(z) (z-\mu_Z) (z-\mu_Z)^T \Sigma_Z^{-1}\right)  \nonumber  \\
		& \quad  + \frac{1}{2}  \tr\left(\int_z p_{X,Y}(z) (z-\mu_Z)(z-\mu_Z)^T\Sigma_{Y,X}^{-1} \right)  \nonumber  \\
		& = - \frac{1}{2}  \tr\left( \Sigma_Z \Sigma_Z^{-1}\right) 
		+ \frac{1}{2}  \tr\left(\Sigma_Z \Sigma_{Y,X}^{-1} \right)  \nonumber \\
		& = - n + n = 0,
	\end{align}
	where we used the definition of covariance matrix of $Z$, i.e., $\Sigma_Z=\int_z p_{X,Y}(z) (z-\mu_Z) (z-\mu_Z)^T$, and 
	\begin{align*}
		\tr\left(\Sigma_Z \Sigma_{U,X}^{-1}\right)
		& = \tr\left(\left[\begin{matrix}
			\Sigma_Y & \cov(Y,X) \\
			\cov(X,Y) & \Sigma_X
		\end{matrix}\right]
		\left[\begin{matrix}
			\Sigma_Y & \bm{0}_n \\
			\bm{0} & \Sigma_X
		\end{matrix}\right]\right) \\
		& = \tr\left(\left[\begin{matrix}
			I_n & \cov(Y,X)\Sigma_X^{-1} \\
			\cov(X,Y)\Sigma_Y^{-1} & I_n
		\end{matrix}\right]\right) \\
		& = \tr(I_n) + \tr(I_n) = 2n.
	\end{align*}
	
	Thus, from \eqref{Eq:MIGTinGaussian1}, \eqref{Eq:MIGTinGaussian2}, and \eqref{Eq:MIGTinGaussian3}, we have \eqref{Eq:MIGTinJointGuassian}. 
	
\end{proof}

The proof of Corollary \ref{Cor:GeneralizationLossViaMI} is based on the Fano's inequality for continuous random variables which is presented in Lemma \ref{Lem:FanoIneqContinuous} for self-containedness.

\begin{lemma}\label{Lem:FanoIneqContinuous}
	(Corollary in \cite{cover_elements_2012}, P255) For an arbitrary random variable $Y$, given side information $X$ and the estimator $\hat{Y}(X)$, it follows that
	\begin{align}\label{Eq:PpltRskLwBd}
		\mbE\left[ (Y - \hat{Y})^2\right] \geq \frac{1}{2\pi e} e^{h(Y|X)}.
	\end{align}
\end{lemma}

Lemma \ref{Lem:FanoIneqContinuous} actually gives the lower bound of the generalization loss, i.e., 
\begin{align*}
	\mcR:=E[(Y-\hat{Y})^2] = \int_{X,Y} p_{X,Y}(x,y) (y - \hat{y}(x))^2 dxdy,
\end{align*}
and it implies that $\mcR\geq \frac{1}{2\pi e}e^{h(Y) - I(X;Y)}$ which is consistent with our intuitions. For example, when $I(X;Y)$ is large (strong dependency between $X$ and $Y$), the lower bound will become small, i.e., easier to give the correct regression labels when the input is given. 

\begin{cor}\label{Cor:GeneralizationLossViaMI}
	(Generalization Loss Lower bound in Terms of Mutual Information for \eqref{Defn:RepresentationLearningDataModel}) We consider a multi-output regression task where the input $Y\in\mbR^n$ of a machine learning systems has all its elements folllowing I.I.D. standard Gaussian distribution, and the output $X\in\mbR^n$ is generated according to \eqref{Defn:RepresentationLearningDataModel}. Then for any estimator $\hat{Y}$ from $X$, we have
	\begin{align}
		\mcR \geq b(n,\rho):=(2\pi e)^{\frac{n-2}{2}} (1-\rho^2)^{\frac{n}{2}}.
	\end{align}
	Moreover,
	\begin{align}\label{Eq:Asymptotic}
		\lim_{n\to\infty} b(n,\rho) = 
		\begin{cases}
			\infty, \ \text{if } \rho\in\left(0,\sqrt{1-\frac{1}{2\pi e}}\right), \\
			\frac{1}{2\pi e},  \ \text{if } \rho = \sqrt{1-\frac{1}{2\pi e}}, \\
			0,  \ \text{if } \rho\in\left(\sqrt{1-\frac{1}{2\pi e}}, 1\right).
		\end{cases}
	\end{align}
\end{cor} 

\begin{proof}\label{Prf:GeneralizationLossViaMI}
	(of Theorem \ref{Cor:GeneralizationLossViaMI}) From the assumptions, we know that $h(Y)=\frac{n}{2}\log(2\pi e)$. Then, we have 
	\begin{align}
		\mcR 
		& \geq \frac{1}{2\pi e} e^{h(Y|X)} \\
		& = \frac{1}{2\pi e} e^{h(Y) - I(X;Y)} \\
		& = \frac{1}{2\pi e} e^{ \frac{n}{2}\log(2\pi e) - \frac{n}{2}\log\frac{1}{1-\rho^2}} \\
		& =  (2\pi e)^{\frac{n-2}{2}} (1-\rho^2)^{\frac{n}{2}}.
	\end{align}
	Since $b(n,\rho) = \frac{\left(2\pi e(1-\rho^2)\right)^{\frac{n}{2}}}{2\pi e}$, then when $2\pi e(1-\rho^2)>1$, we have $b(n,\rho)$ goes to $\infty$ as $n$ goes to infinity. This gives $\rho\in(0,\sqrt{1-\frac{1}{2\pi e}})$. Similarly, we can get the other scenarios in \eqref{Eq:Asymptotic}.
\end{proof}

\section{Training and Inference Pipeline of Mutual Information Learned Regressor}

In this section, we present the training and the inference pipelines for regression under the mutual information based supervised learning framework in Figure \ref{Fig:MILR_Training} and \ref{Fig:MILR_Inference}. During the training process, a batch of data $\{(x_i,y_i)\}_{i=1}^B$ is sampled to obtain the empirical probability mass function (pmf) estimates $\hat{P}_{Y|X}$, $\hat{P}_X$, and $\hat{P}_Y$ associated with the conditional label probability density function (pdf) $p_{Y|X}$, the marginal input probability density function $p_{X}$, and the marginal label pdf $p_{Y}$, respectively. Besides, the data batch is fed to a neural network parameterized by $\theta$ to produce the pdf parameters $\{(\mu_i(x_i; \theta_{Y|X}), \sigma_i(x_i; \theta_{Y|X}))\}_{i=1}^B$ associated with a learned conditional label pdf and a learned marginal label pdf. As an example, for each data point $(x_i,y_i)\in\mbR^n\times\mbR$, we can assume the conditional label pdf $q_{Y|x_i}(y_i|x_i)$ follows a Gaussian distribution $\mcN(\mu_i,\sigma_i^2)$ with mean $\mu_i\in\mbR$ and variance $\sigma_i^2\in\mbR_+$, and the $f_\theta(x_i)\in\mbR^2$ gives estimate of $[\mu_i\ \sigma_i]^T$. The empirical estimate $\hat{P}_{Y|X}$ and the learned estimate $q_{Y|X}$ will be used to calculate the conditional cross entropy $H(\hat{p}_{Y|X}, q_{Y|X})$ whose infimum is an estimate of the conditional differential entropy $h({Y|X})$. The $q_{Y|X}$ will be combined with $\hat{p}_X$ to obtain  a learned marginal label distribution $q_{Y;\theta_Y|X}$, and we then calculate the cross entropy $H(\hat{p}_Y, q_{Y;\theta_{Y|X}})$ whose infimum is an estimate of the $h({Y})$. Finally, the mutual information learning (MIL) loss will be 
\begin{align}\label{Eq:milLoss_MIForm}
	\hat{I}(X;Y): = \min_{\theta_{Y|X}} H(\hat{p}_Y,q_{Y;\theta_{Y|X}}) - \min_{\theta_{Y|X}} H(\hat{p}_{Y|X},q_{Y|X; \theta_{Y|X}}).
\end{align}
By solving \eqref{Eq:milLoss_MIForm}, we can use the learned label conditional distribution $q_{Y|X; \theta_{Y|X}}$ for inference by following a maximum likelihood rule as in Figure \ref{Fig:MILR_Inference}. For example, when we assume the $q_{Y|x_i}$ follows a $\mcN(\mu_i, \sigma^2_i)$, then the prediction of $y_i$ when $x_i$ is given will be $\mu_i$, i.e., $\hat{y}_i:=\mu_i$. Though we followed Yi et al. to use parameter sharing so that only a single neural network is needed \cite{yi_mutual_2022}, it is possible to use two separate neural networks for approximating $q_{Y|X}$ and $q_{Y}$ separately, i.e., 
\begin{align}\label{Eq:milLoss_MIForm_WoParamShr}
	\hat{I}(X;Y): = \min_{\theta_{Y}} H(\hat{p}_Y,q_{Y;\theta_{Y}}) - \min_{\theta_{Y|X}} H(\hat{p}_{Y|X},q_{Y|X; \theta_{Y|X}}).
\end{align}

\begin{figure}[!htb]
	\centering 
	\includegraphics[width=\linewidth]{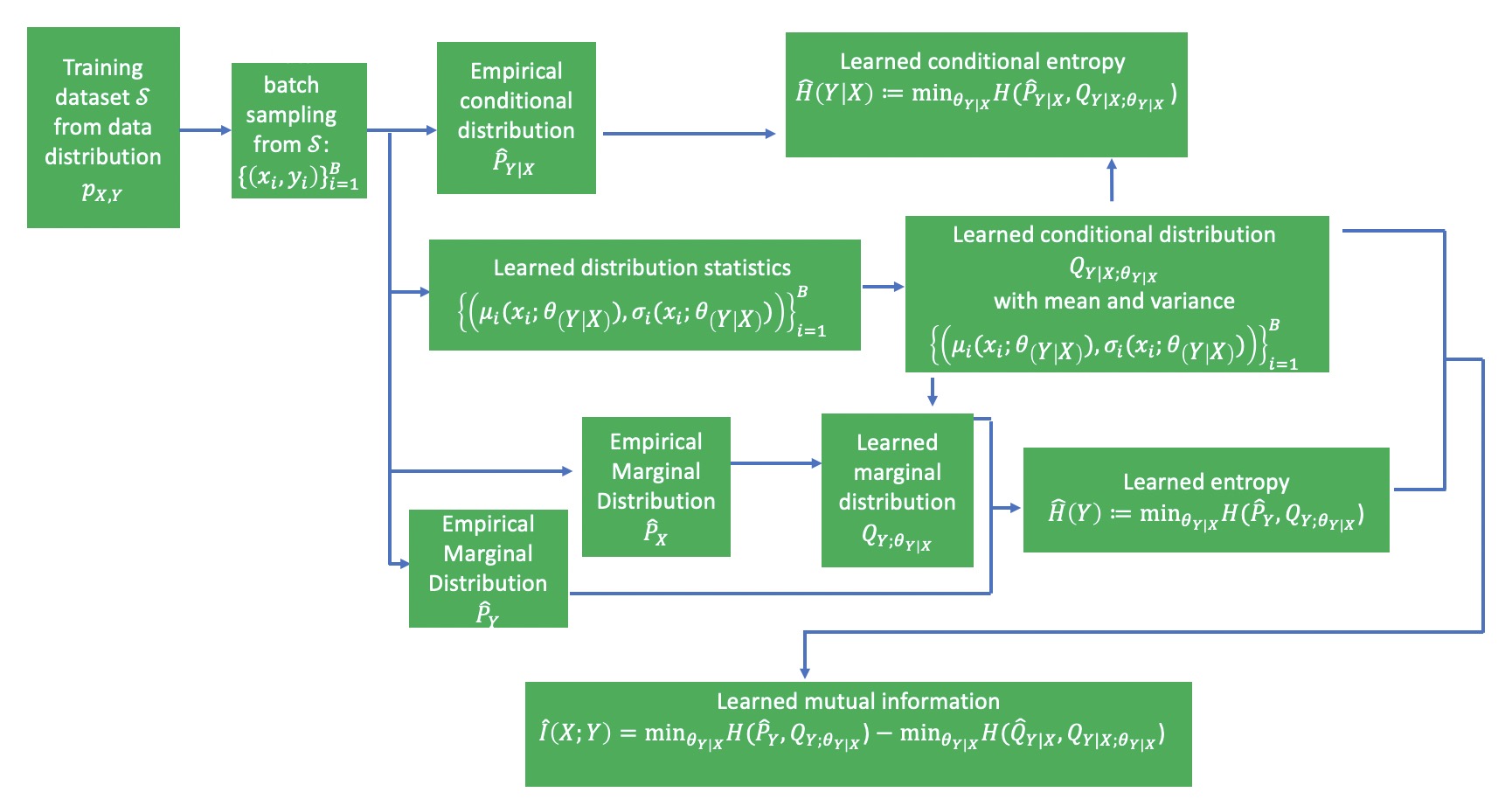}
	\caption{Mutual information learned regression (MILR) framework: we assumed that the probability density distribution is determined by its mean and variance.}
	\label{Fig:MILR_Training}
\end{figure}

\begin{figure}[!htb]
	\centering 
	\includegraphics[width=\linewidth]{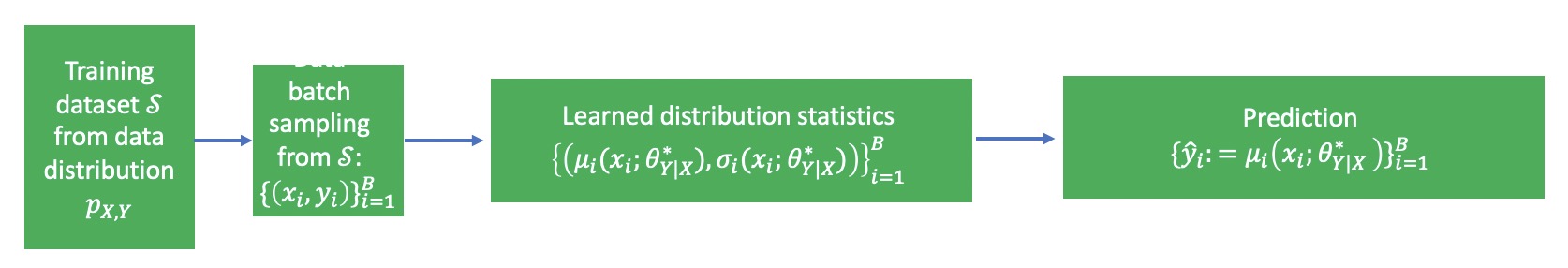}
	\caption{Mutual information learned regression (MILR) framework: we assumed the mode of the probability density function is its mean.}
	\label{Fig:MILR_Inference}
\end{figure}

\end{document}